\documentclass[11pt]{article}
\usepackage{fullpage}
\usepackage{amsmath,amssymb,graphicx,url,amsthm}
\usepackage[algo2e,ruled]{algorithm2e}
\newtheorem{theorem}{Theorem}[section]
\usepackage[rgb]{xcolor}
\usepackage{hyperref}
\newcommand{\citep}[1]{\cite{#1}}

\title{Near--Optimal Correlation Clustering with Privacy}
\usepackage{times}

\author{%
 Vincent Cohen-Addad \\ Google Research  \\ \texttt{cohenaddad@google.com}
 \and
 Chenglin Fan \\ Sorbonne University  \\ \texttt{fanchenglin@gmail.com}
 \and
 Silvio Lattanzi \\ Google Research  \\ \texttt{silviol@google.com}
 \and
 Slobodan Mitrović \\ UC Davis  \\ \texttt{slobo@mit.edu}
 \and
 Ashkan Norouzi-Fard \\ Google Research  \\ \texttt{ashkannorouzi@google.com}
 \and
 Nikos Parotsidis \\ Google Research  \\ \texttt{nikosp@google.com}
 \and
 Jakub Tarnawski \\ Microsoft Research  \\ \texttt{jakub.tarnawski@microsoft.com}%
}

 	\DontPrintSemicolon
   \SetKwInOut{Input}{Input}
   \SetKwInOut{Output}{Output}
\usepackage{thmtools,thm-restate}
\usepackage[shortlabels]{enumitem}

\usepackage[capitalize, nameinlink]{cleveref}
\crefname{theorem}{Theorem}{Theorems}
\crefname{lemma}{Lemma}{Lemmas}
\Crefname{invariant}{Invariant}{Invariants}
\Crefname{claim}{Claim}{Claims}
\Crefname{observation}{Observation}{Observations}
\Crefname{algorithm}{Algorithm}{Algorithms}
\Crefname{figure}{Figure}{Figures}

\newtheorem{lemma}[theorem]{Lemma}

\newtheorem{definition}[theorem]{Definition}
\newtheorem{fact}[theorem]{Fact}

\newtheorem{claim}[theorem]{Claim}

\DeclareMathOperator{\Lap}{Lap}
\DeclareMathOperator{\cost}{cost}

\newcommand{\rb}[1]{\left( #1 \right)}
\newcommand{\bbR}{\mathbb{R}}

\newcommand{\cE}{\mathcal{E}}
\newcommand{\cA}{\mathcal{A}}
\newcommand{\cM}{\mathcal{M}}
\newcommand{\cS}{\mathcal{S}}
\newcommand{\cG}{\mathbb{G}}
\newcommand{\eps}{\epsilon}
\newcommand{\prob}[1]{\Pr \left[ #1 \right]}
\newcommand{\hd}{\hat{d}}
\newcommand{\hl}{\hat{l}}
\newcommand{\hG}{\hat{G}}
\newcommand{\myvec}[1]{\overline{#1}}

\newcommand{\AlgCC}{\textsc{Alg-CC}\xspace}
\newcommand{\Erem}{E_{\text{rem}}}
\newcommand{\tvlight}{TV-light\xspace}
\newcommand{\tvdis}{TV-disagree\xspace}
\newcommand{\tG}{\tilde{G}}
\newcommand{\tN}{\tilde{N}}

\definecolor{darkgreen}{rgb}{0.2,0.7,0.2}

\begin{document}

\maketitle

\begin{abstract}
Correlation clustering is a central problem in unsupervised learning, with applications spanning community detection, duplicate detection, automated labelling and many more. In the correlation clustering problem one receives as input a set of nodes and for each node a list of co-clustering preferences, and the goal is to output a clustering that minimizes the disagreement with the specified nodes' preferences. In this paper, we introduce a simple and computationally efficient algorithm for the correlation clustering problem with provable privacy guarantees. Our approximation guarantees are stronger than those shown in prior work and are optimal up to logarithmic factors.
\end{abstract}

\section{Introduction}
Clustering is a central problem in unsupervised machine learning. The goal of clustering is to partition a set of input objects so that similar objects are assigned to the same part while dissimilar objects are assigned to different parts of the partition. Clustering has been extensively studied throughout the years and many different formulations of the problem are known. In this paper we study the classic correlation clustering problem in the context of differential privacy. 

In the correlation clustering problem~\citep{bansal2004correlation} one gets a graph whose vertices are the objects to be clustered and whose edges represent clustering preferences between the objects. More specifically, the input of the problem is a (possibly edge-weighted) graph with positive and negative labels on the edges such that positive edges represent similarities between vertices and negative edges represent dissimilarities. Then the correlation clustering objective asks to minimize the sum of (the weights of) positive edges across clusters plus the sum of (the weights) of negative edges within clusters. Thanks to its simple and elegant formulation, the problem has received much attention and it is used to model several  practical applications including finding clustering ensembles \citep{bonchi2013overlapping}, duplicate detection \citep{arasu2009large}, community mining \citep{chen2012clustering},
 disambiguation tasks \citep{kalashnikov2008web}, automated labelling \citep{agrawal2009generating, chakrabarti2008graph} and many more. In this paper we focus on the most studied version of the problem where all edges have unit weight. In this case the best known algorithm~\citep{chawla2015near} has an approximation guarantee of $2.06$, which improves over a 2.5-approximation due to \citep{ailon2008aggregating}. 
 Furthermore, when the number of clusters is
upper-bounded by $k$, a polynomial-time approximation scheme is known~\citep{giotis2005correlation}. 
In the weighted case  a $O(\log n)$-approximation is known~\citep{demaine2006correlation}, and improving upon this
would lead to a better approximation algorithm for the notoriously difficult multicut problem.
For the maximization
version of the problem, where the goal is to maximize the sum of (the weights of) the positive
edges within clusters plus the sum of (the weights of) the negative 
edges across clusters, \cite{charikar2005clustering,swamy2004correlation} gave a  $0.7666$-approximation algorithm for the weighted case and a PTAS is known for the unweighted
case~\citep{bansal2004correlation}.
 
{\em Differential Privacy (DP)} is the \emph{de facto} standard for  user privacy~\citep{DBLP:conf/tcc/DworkMNS06}, and it is of fundamental importance to design solutions for classic unsupervised problems in this setting. In differential privacy, the focus is on datasets, $G$ and $G'$, that differ on a single entry.  An algorithm $\cA$ is $(\epsilon,\delta)$-\emph{differentially private} if the probabilities of observing any  set of possible outputs $S$ of $\cA$ when run on two ``almost identical'' inputs $G$ and $G'$ are similar:
$\prob{\cA(G) \in S} \leq e^{\epsilon} \cdot \prob{\cA(G') 
\in S} + \delta$.
Over the last decade, there have been many works considering problems related to private graphs, e.g.,~\citep{DBLP:conf/icdm/HayLMJ09, DBLP:conf/pods/RastogiHMS09, DBLP:conf/soda/GuptaLMRT10,DBLP:conf/tcc/GuptaRU12,DBLP:conf/innovations/BlockiBDS13,DBLP:journals/pvldb/KarwaRSY11,DBLP:conf/tcc/KasiviswanathanNRS13,DBLP:conf/focs/BunNSV15,DBLP:conf/nips/AroraU19, DBLP:conf/nips/UllmanS19,DBLP:conf/focs/BorgsCSZ18,DBLP:conf/soda/EliasKKL20,bun2021differentially, DBLP:conf/icml/NguyenV21,DBLP:journals/corr/abs-2106-00508}. We briefly review two possible definitions of privacy in graphs.

 {\bf Edge Privacy.} In the edge privacy setting, two datasets are considered to be close if they differ on a single edge. \cite{DBLP:conf/icdm/HayLMJ09} introduced a differentially edge-private algorithm for releasing the degree distribution of a graph. 
 They also proposed the notion of differential node privacy and highlighted some of the difficulties in achieving it.
\cite{DBLP:conf/tcc/GuptaRU12} showed how to answer cut queries in a private edge model. \cite{DBLP:conf/focs/BlockiBDS12} improved the error for small cuts.  \cite{DBLP:conf/soda/GuptaLMRT10} showed
how to privately release a cut close to the optimal error size.  \cite{DBLP:conf/nips/AroraU19} studied the private sparsification of graphs, which was exemplified by a proposed graph meta-algorithm for privately answering cut-queries with improved accuracy.
\cite{DBLP:conf/soda/EliasKKL20} studied the problem of 
private synthetic graph release while preserving all cuts.
Recently, \cite{DBLP:conf/icml/NguyenV21,DBLP:journals/corr/abs-2106-00508} proposed frameworks for the private densest subgraph problem.  

 \textbf{Node Privacy.} Node differential privacy
 requires the algorithm to hide the presence or absence of a single node and the (arbitrary) set of edges incident to that node. However, node-DP is often difficult to achieve without compromising accuracy, because even very simple graph statistics can be highly sensitive to adding or removing a single node~\citep{DBLP:conf/tcc/KasiviswanathanNRS13, DBLP:conf/innovations/BlockiBDS13, DBLP:conf/focs/BorgsCSZ18}. 

Differentially private clustering has been extensively studied~\cite{balcan,anamayclustering,badih_approximation,badih_local, clustering_with_convergence}. Nevertheless, until very recently
no theoretical results were known for differentially private correlation clustering. In this context, the privacy is on the edges of the graph: namely, two graphs on the same set of vertices are \emph{adjacent} if they differ by exactly one edge.

In a recent work, \cite{bun2021differentially} obtained the first differentially private correlation clustering algorithm
with approximation guarantees using differentially private synthetic graph release~\cite{DBLP:conf/tcc/GuptaRU12,DBLP:conf/soda/EliasKKL20}. The framework
proposed by~\cite{bun2021differentially} is very elegant and general as it allows one to run any non-private correlation clustering approximation algorithm on a modified version of the input graph that ensures privacy. Namely,
any $\alpha$-approximation algorithm to correlation clustering leads to a differentially private approximation algorithm with multiplicative approximation $\alpha$ and additive approximation $O(n^{1.75})$. This applies to the more general weighted version of the problem. 
In the same paper, the authors also obtain an $\Omega(n)$ lower bound on the additive approximation of differentially private algorithms, even on unweighted graphs that consist of a single path.
However, the framework from \citep{bun2021differentially} is rather impractical, and the additive error is far from matching the lower bound on the additive approximation.
These results prompt the natural question of
determining the best approximation guarantees
(multiplicative and additive)
that are possible under differential privacy.

As observed by~\cite{bun2021differentially},
instantiating the exponential mechanism~\cite{conf/focs/McSherryT07}
over the space of all clusterings
yields an algorithm with additive error $O(n \log n)$.
However, it is not known how to efficiently sample from the exponential mechanism for correlation clustering.
\cite{bun2021differentially} state finding a polynomial-time algorithm that matches the additive error of the exponential mechanism as an "exciting open problem given the prominent
position correlation clustering occupies both in theory and practice".




\paragraph{Our Results and Technical Overview.} In this paper we present a new differentially private algorithm for the unweighted setting which achieves a constant multiplicative approximation and a nearly optimal $O(n \log^2 n)$ additive approximation. 
More precisely, we show:
\begin{theorem} \label{thm:main}
    For any $\epsilon$ and $\delta$
    there is an algorithm for min-disagree correlation clustering
    on unweighted complete graphs
    that is $(\epsilon,\delta)$-differentially private
    and returns a solution of cost at most $O(1) \cdot \mathrm{OPT} + O\left(\frac{n \log^2 n \log(1/\delta)}{\epsilon^2}\right)$.
\end{theorem}

Our algorithm is given in \cref{sec:algorithm} as \cref{alg:main}.
Its privacy is proved in \cref{sec:analysis_of_privacy} (\cref{thm:privacy_main}),
and \cref{sec:approximation_analysis} is devoted to the approximation guarantees (\cref{thm:approx_main}).

The lower bound on the additive approximation in~\cite{bun2021differentially}
does not preclude an $(\epsilon,\delta)$-DP algorithm for correlation clustering on complete graphs with sublinear error.
More precisely, \cite{bun2021differentially} show an $\Omega(n/\epsilon)$ lower bound for weighted paths and an $\Omega(n)$ lower bound for unweighted paths --
both non-complete graphs --
against pure $(\epsilon,0)$-DP algorithms.
Nevertheless, we prove that linear error is necessary even for complete unweighted graphs and $(\epsilon,\delta)$-privacy, showing that our algorithm is indeed near-optimal.

\begin{restatable}{theorem}{lowerbound} \label{thm:lower_bound}
    Any $(\epsilon,\delta)$-DP algorithm for correlation clustering on unweighted complete graphs has additive error $\Omega(n)$,  assuming $\epsilon \le 1$ and $\delta \le 0.1$.
\end{restatable}
The proof of \cref{thm:lower_bound} is given in \cref{sec:lower_bound}.

Our approach follows the recent result of~\cite{cohen2021correlation} for obtaining the first $O(1)$-rounds, $O(1)$-approximation algorithm for correlation clustering in the parallel setting. However, to obtain our bounds in the context of differential privacy we introduce several critical new ideas.

At a high level, the algorithm of~\cite{cohen2021correlation} trims the input graph in two steps. First, it only keeps the positive edges between those vertices that have very similar neighborhoods of positive edges (such pairs of vertices are said to be \emph{in agreement}). More precisely, for two nodes to be in agreement the 
size of the intersection of the positive neighborhoods should be  larger than some threshold $\beta$ times the positive degrees of each endpoint.
Second, it removes the positive edges whose endpoints have lost a significant fraction of its positive neighbors during the first step (such a vertex is called a \emph{light} vertex). Finally, the resulting clusters are given by the connected components induced by the remaining positive edges.

Our approach consists in making each of the above steps (agreement computation, light vertex computation, connected components) differentially private.

A natural way to make the agreement computation  differentially private is to add Laplace noise to the size of the intersection of the neighborhoods for each pair of vertices $u,v$ and to decide that $u,v$ are in agreement if the noisy intersection size is larger than $\beta$ times the positive degrees of $u$ and $v$.
One of the crucial challenges here is to make sure that the total amount of noise needed to make the entire procedure
differentially private is bounded,
so that we can still obtain strong approximation guarantees.

The second step, the computation of light vertices, can be made differentially private in a very natural way: simply add Laplace noise to the degree of each vertex after the removal of the edges whose endpoints are not in agreement and decide whether a vertex is light based on the noisy degree.

The third step, the connected components computation, is the most challenging. Here we need to argue that computing  connected components of the graph induced by the positive edges  is differentially private.
In other words, we have to show that the graph induced by these edges has no ``bridge''; that is, there is not a single positive edge whose removal would increase the number of  connected components.
This is done by establishing new properties of the algorithm and showing that if all the previous steps have succeeded, the presence of a bridge is in fact a very unlikely event.
Moreover, to guarantee privacy, we must carefully modify the way we treat light vertices, as well as those of low degree.

\paragraph{Discussion of Recent Work.}
We note that the work of \cite{cohen2021correlation} has been followed by the work of \cite{DBLP:conf/innovations/Assadi022}, who improved it in the context of streaming algorithms; however, it is not clear that there would be a benefit in using the framework of \cite{DBLP:conf/innovations/Assadi022}, either in terms of running time, multiplicative, or additive approximation.

In concurrent and independent work, Liu~\cite{Daogao2022} proposed an $(\epsilon,\delta)$-DP algorithm that  achieves a multiplicative approximation of $O(\log n)$ and an additive error of $\tilde{O}(n^{1.5})$ for general weighted graphs (assuming constant $\epsilon$ and $\delta$),
improving upon the $O(n^{1.75})$ error of~\cite{bun2021differentially}.
This result arises via a more careful analysis of differentially private synthetic graph release in the case of correlation clustering.
Furthermore, for unweighted complete graphs Liu obtained an algorithm with constant multiplicative approximation and an additive error of $O(n \log^4 n \cdot \sqrt{\Delta^* + \log n})$, where $\Delta^*$ is the maximum degree of positive edges in the graph.
The latter is a pivot-based algorithm augmented with Laplace noise.
This results in an $O(n^{1.5}\log^4 n)$ worst-case additive error.
Our algorithm yields an additive error of $O(n \log^2 n)$ (\cref{thm:main}),
which significantly improves upon the result of Liu.

\section{Preliminaries} \label{sec:prelim}

\paragraph{Correlation clustering.}
In the min-disagree variant of the correlation clustering problem in complete graphs one receives as input a complete signed graph $G = (V,E^+, E^-)$, where $E^+$ (resp., $E^-$) denotes the set of ``+'' edges (resp., ``-''), and the objective is to compute a clustering $\mathcal{C}=\{C_1, \dots, C_t\}$ of $V$ that minimizes the number of ``-" edges whose endpoints are part of the same cluster plus the number of ``+" edges whose endpoints belong to the distinct clusters.

In the sequel we will use $G=(V, E)$ to refer to $G = (V,E^+, E^-)$ (with $E = E^+$).

\paragraph{Differential privacy.}
The definition of differential privacy (DP) is due to~\cite{DBLP:conf/tcc/DworkMNS06};
we use the 
precise formulation introduced by~\cite{DBLP:conf/icalp/Dwork06}.
\begin{definition}[Differential Privacy] \label{def:dp}
	A mechanism (randomized algorithm) $M$
	with domain $\cG$ and range $\cM$,
	which we will write as $M : \cG \to \cM$,
	is $(\epsilon,\delta)$-differentially private if for any two adjacent datasets $G, G' \in \cG$ and set of outcomes $S \subseteq \cM$ we have
	\begin{equation}
	    \label{eq:def_dp}
    	\prob{M(G) \in S} \leq e^\epsilon \cdot \prob{M(G')\in S} + \delta \,. 
	\end{equation}
\end{definition}
Recall that in this work, two datasets (graphs) are adjacent if they have the same set of vertices and differ by one edge.


An important property of differential privacy is that we can compose multiple differentially private subroutines into a larger DP algorithm with privacy guarantees.

\begin{lemma}[\cite{dwork2014algorithmic}, Theorem B.1] \label{lem:composition}
Let $M_1 : \cG \to \cM_1$ be a randomized algorithm that is $(\epsilon_1,\delta_1)$-DP.
Further let $M_2 : \cG \times \cM_1 \to \cM_2$ be a randomized algorithm such that for every fixed $m_1 \in \cM_1$, the mechanism $\cG \ni G \mapsto M_2(G, m_1) \in \cM_2$ is $(\epsilon_2,\delta_2)$-DP.
Then the composed mechanism $\cG \ni G \mapsto M_2(G, M_1(G)) \in \cM_2$ is $(\epsilon_1 + \epsilon_2, \delta_1 + \delta_2)$-DP.
\end{lemma}
We also use the following property to analyse the privacy of our algorithm.
\begin{restatable}{lemma}{lowprobbad} \label{lem:low_prob_bad}
Let $M_1 : \cG \to \cM_1$ be a randomized algorithm that is $(\epsilon,\delta)$-DP.
Suppose $B \subseteq \cM_1$ is a set of "bad outcomes" with $\prob{M_1(G) \in B} \le \delta^*$
for any $G \in \cG$.
Further let $M_2 : \cG \times \cM_1 \to \cM_2$ be a \emph{deterministic} algorithm such that for every fixed "non-bad" $m_1 \in \cM_1 \setminus B$
we have $M_2(G,m_1) = M_2(G',m_1)$ for adjacent $G,G' \in \cG$.
Then the composed mechanism $\cG \ni G \mapsto M_2(G, M_1(G)) \in \cM_2$ is $(\epsilon, \delta + \delta^*)$-DP.
\end{restatable}
The proof is routine:
\begin{proof}
Fix $G,G' \in \cG$ and a set of outcomes $S_2 \subseteq \cM_2$.
Define
\[
    S_1^* := \{ m_1 \in \cM_1 \setminus B : M_2(G,m_1) \in S_2 \} \,.
\]
By assumption we have
\begin{equation} \label{eq:M2GM2G}
    S_1^* = \{ m_1 \in \cM_1 \setminus B : M_2(G',m_1) \in S_2 \} \,.
\end{equation}
Now we can write
\begin{align*}
    \prob{M_2(G,M_1(G)) \in S_2} &\le  \prob{M_1(G) \in B} + \prob{M_1(G) \not \in B \text{ and } M_2(G,M_1(G)) \in S_2} \\
    &\le \delta^* + \prob{M_1(G) \in S_1^*} \\
    &\overset{\mathrm{DP}}{\le} \delta^* + e^\epsilon \cdot \prob{M_1(G') \in S_1^*} + \delta \\
    &\overset{\eqref{eq:M2GM2G}}{=} \delta^* + e^\epsilon \cdot \prob{M_1(G') \not \in B \text{ and } M_2(G',M_1(G')) \in S_2} + \delta \\
    &\le \delta^* + e^\epsilon \cdot \prob{M_2(G',M_1(G')) \in S_2} + \delta  \,.
\end{align*}
\end{proof}

If we have $k$ mechanisms that are $(\epsilon,\delta)$-DP, their (adaptive) composition is $(k \epsilon, k \delta)$-DP. However, it is also possible to reduce the linear dependency on $k$ in the first parameter to roughly $\sqrt{k}$ by accepting higher additive error.

\begin{theorem} [Advanced Composition Theorem~\cite{DBLP:conf/focs/DworkRV10}] \label{thm:advanced_composition}
For all $\epsilon, \delta'  \geq 0$, an adaptive composition of $k$
$(\epsilon,0)$-differentially private mechanisms satisfies $(\epsilon', \delta')$-differential
privacy  for
\[
\epsilon' = \sqrt{2 k \ln(1/\delta')} \epsilon + k \epsilon (e^\epsilon - 1) \,.
\]
\end{theorem}

Let $\Lap(b)$ be the Laplace distribution with parameter $b$ and mean $0$.
We will use the following two properties of Laplacian noise.

\begin{fact}\label{lem:folk}
    Let $Y \sim \Lap(b)$ and $z > 0$. Then
    \[
        \prob{Y > z} = \frac{1}{2} \exp{\left (-\frac{z}{b} \right )}
    \qquad \text { and} \qquad
        \prob{|Y| > z} =  \exp{\left (-\frac{z}{b} \right )}.
    \]
\end{fact}

\begin{theorem}[\cite{dwork2014algorithmic}, Theorem 3.6]
    \label{theorem:mechanism}
    Let $f : \cG \to \mathbb{R}^k$ be a function.
    Denote by $\Delta f$
    its \emph{$\ell_1$-sensitivity},
    which is the maximum value of $\| f(G) - f(G') \|_1$ over adjacent datasets $G, G' \in \cG$.
    Then $f + (Y_1, ..., Y_k)$,
    where the variables $Y_i \sim \Lap(\Delta f / \epsilon)$ are iid,
    is $(\epsilon,0)$-DP.
\end{theorem}

\paragraph{Graph notation.}
Given an input graph $G$ and a vertex $v$, we denote its set of neighbors by $N(v)$ and its degree by $d(v) = |N(v)|$.
As in~\citep{cohen2021correlation},
we adopt the convention that $v \in N(v)$ for every $v \in V$ (one can think that we have a self-loop at every vertex that is not removed at any step of the algorithm).

\section{Private Algorithm for Correlation Clustering}
\label{sec:algorithm}

\newcommand{\epsagr}{\epsilon_{\mathrm{agr}}}
\newcommand{\deltagr}{\delta_{\mathrm{agr}}}

In this section we formally define our algorithm, whose pseudo code is available in \cref{alg:main} together with \cref{definition:noised-agreement}.
The algorithm uses a number of constants, which we list here for easier reading and provide feasible settings for their values:
\begin{itemize}[itemsep=-1.5pt,topsep=10pt]
    \item $\epsilon > 0$ and $\delta \in (0,\frac12)$ are user-provided privacy parameters.
    \item $\beta$ and $\lambda$, used the same way as in \cite{cohen2021correlation}, parametrize the notions of agreement (\cref{definition:noised-agreement}) and lightness (\cref{line:light}), respectively.
    For the privacy analysis, any $\beta, \lambda \le 0.2$ would be feasible.
    These parameters also control the approximation ratio, which is $O(1/(\beta \lambda))$ assuming that $\beta$, $\lambda$ are small enough; as in \cite{cohen2021correlation}, one can set e.g.~$\beta = \lambda = 0.8/36 \approx 0.02$.
    \item $\epsagr$, $\deltagr$ and $\gamma$ are auxiliary parameters that control the noise used in computing agreements; they are functions of $\epsilon$ and $\delta$ and are defined in \cref{definition:noised-agreement}.
    \item $\beta'$ and $\lambda'$ are used in the privacy analysis; both can be set to $0.1$.
    \item $T_0$ is a degree threshold; we return vertices whose (noised) degree is below that threshold as singletons,
    which incurs an additive loss of $O(T_0 n)$ in the approximation guarantee.
    We set \[ T_0 = T_1 + \frac{8\log (16/\delta)}{\epsilon} \,, \]
    where for the privacy analysis we require $T_1$ to be a large enough constant;
    namely, one can take the maximum of the right-hand sides of~\eqref{eq:T1_large_enough_1}, \eqref{eq:T1_large_enough_2}, \eqref{eq:T1_large_enough_3}, \eqref{eq:T1_large_enough_4}, \eqref{eq:T1_large_enough_5}, \eqref{eq:T1_large_enough_6}, \eqref{eq:T1_large_enough_7}, and \eqref{eq:T1_large_enough_8};
    asymptotically in terms of $\epsilon$ and $\delta$, this is of the order $O\left( \ln(1/(\epsilon \delta))^2 \ln(1/\delta) \epsilon^{-2} \right)$ (assuming $\epsilon \le O(1)$).
    To additionally obtain a constant-factor approximation guarantee, we further require a polylogarithmic $T_1$, namely
    of the order $O(\log^2 n \log(1/\delta) \epsilon^{-2})$
    -- see Eq.~\eqref{eq:constraint-on-T_1} in the proof of 
    \cref{thm:approx-proof}.
\end{itemize}

The following notion is central to our algorithm and is used as part of \cref{line:noised_agreement} of \cref{alg:main}.

\begin{definition}[Noised Agreement]
\label{definition:noised-agreement}
Let us define $\epsagr = \epsilon / 5.8$, $\deltagr = \delta / 9.6$, and
\[ \gamma = \frac{\sqrt{\frac{4 \epsagr}{\ln(1/\deltagr)} + 1} + 1}{\sqrt{2}}.\]
(Note that $\gamma \ge \sqrt{2}$.)
Further
let
$H$ be as defined in \cref{line:noised_degree} of \cref{alg:main}.
For each pair of vertices $u, v \in H$,
let $\cE_{u,v}$ be an
independent random
variable such that 
\[\cE_{u,v} \sim \Lap\left( \max \left( 1, \frac{\gamma \sqrt{\max(5,d(u),d(v)) \cdot \ln (1/\deltagr)}}{\epsagr} \right) \right).\]
We say that two vertices $u \ne v \in H$ are in \emph{$i$-noised  agreement} if $|N(u)\triangle N(v)|+\cE_{u,v} < i\beta \cdot 
\max(d(u), d(v))$.

If $u$ and $v$ are in $1$-noised  agreement,
we also say that $u$ and $v$ are in noised agreement; otherwise they are not in noised agreement.
\end{definition}
Note that we do not require that $(u,v) \in E$, although \cref{alg:main} will only look at agreement of edges.

\begin{algorithm2e}
\LinesNumbered
  \Input{$G=(V,E)$: a graph \\
		$\eps$, $\delta$: privacy parameters}

	Let $\hd(v)=d(v)+Z_v$ denote the \emph{noised degree}  of $v$, where $Z_v \sim \Lap(8/\epsilon)$.
	Let $H = \{v \in V : \hd(v) \ge T_0 \}$ denote the set of high-degree vertices.
	\label{line:noised_degree}
	
	Discard from $G$ the edges that are not in noised agreement (see \cref{definition:noised-agreement}). (First compute the set of these edges. Then remove this set. Note that this includes all edges with an endpoint not in $H$.)
    \label{line:noised_agreement}
   
    Let $l(v)$ be the number of edges incident to $v$ discarded in the previous step, and define $\hat{l}(v)=l(v)+Y_v$, where $Y_v \sim \Lap(8/\epsilon)$. Call a vertex $v$ \emph{light} if $\hat{l}(v)>\lambda d(v)$, and  otherwise call $v$ \emph{heavy}.
    \label{line:light}
 
    Discard all edges whose both endpoints are light.
    Call the current graph  $\hG$, or the \emph{sparsified graph}. Compute its connected components.
    Output  the heavy vertices in each component $C$ as a cluster. Each light vertex  is output as a singleton.
    \label{line:component}
\label{line:light-to-singletons}

  \caption{Private correlation clustering  using noised agreement}
	\label{alg:main}
\end{algorithm2e}

\section{Analysis of Privacy}
\label{sec:analysis_of_privacy}

Our analysis proceeds by fixing two adjacent datasets $G$, $G'$ (i.e., $G$ and $G'$ are graphs on the same vertex set that differ by one edge)
and analyzing the privacy loss of each step of the algorithm.
We compose steps up to \cref{line:light} by repeatedly invoking \cref{lem:composition}, which allows us to assume, when analyzing the privacy of a step, that the intermediate outputs (states of the algorithm) up to that point are the same between the two executions.
Finally, we use \cref{lem:low_prob_bad} to argue that
if the intermediate outputs before \cref{line:component} are the same,
then
the output of \cref{line:component} (i.e., the final output of \cref{alg:main}) does not depend on whether the input was $G$ or $G'$,
except on a small fragment of the probability space that we can charge to the additive error $\delta$.

We begin the analysis by reasoning about \cref{line:noised_degree}.

\begin{lemma}
    \label{lem:privacy_of_noised_degrees}
    Consider \cref{line:noised_degree} as a randomized algorithm that outputs $H$.
    It is $(\epsilon/4,0)$-DP.
\end{lemma}
\begin{proof}
    The sensitivity $\Delta d$ of the function $d$ is $2$, as adding an edge changes the degree of two vertices by $1$.
    Therefore, by \cref{theorem:mechanism}, $\hd$ is $(\epsilon/4,0)$-DP.
    Furthermore, $H$ is a  function that only depends on the input (set of edges) deterministically via $\hd$.
\end{proof}

Now we fix adjacent $G$, $G'$;
in this proof we will think that the domain of \cref{alg:main} is $\cG = \{G,G'\}$
and the notion of ``$(\cdot,\cdot)$-DP'' always refers to these two fixed inputs.
Denote $\{(x,y)\} = E(G) \triangle E(G')$
to be the edge on which $G$ and $G'$ differ.

\begin{restatable}{lemma}{privacyofagreement}
    \label{lem:privacy_of_agreement}
    Under fixed $H$,
    consider \cref{line:noised_agreement} as a randomized algorithm that, given $G$ or $G'$, outputs the noised-agreement status of all edges in $E(G) \cup E(G')$.
    It is $(2.9 \epsagr, 2.4 \deltagr)$-DP.
\end{restatable}
Note that \cref{alg:main} only computes the noised-agreement status of edges that are present in its input graph (which might be the smaller of the two), but without loss of generality we can think that it computes the status of all edges in $E(G) \cup E(G')$
(and possibly does not use this information for the extra edge).

In the proof of \cref{lem:privacy_of_agreement} we apply the Advanced Composition Theorem 
to the noised-agreement status of edges incident on $x$ or $y$
to argue that the Laplacian noise $\cE_{u,v}$ of magnitude roughly $\sqrt{\max(d(u),d(v))}$ is sufficient.
\begin{proof}
    We will show that the sequence 
    \begin{equation} \label{eq:sequence}
        (|N(u) \triangle N(v)| + \cE_{u,v} - \beta \cdot \max(d(u),d(v)) : (u,v) \in E(G) \cup E(G'))
    \end{equation}
    (which determines the noised-agreement status of these edges)
    has the desired privacy guarantee.

    Define $E_x$ to be those edges in $E(G) \cup E(G')$ that are adjacent to $x$, but are not $(x,y)$, and $E_y$ similarly. The sequence \eqref{eq:sequence} can be decomposed into \textbf{four parts} (with independent randomness):
    on $(E(G) \cup E(G')) \setminus (E_x \cup E_y)$,
    on $E_x$,
    on $E_y$,
    and
    on ${(x,y)}$.
    
    For the \textbf{first part}, the function
    \begin{equation*}
        (|N(u) \triangle N(v)| - \beta \cdot \max(d(u),d(v)) : (u,v) \in (E(G) \cup E(G')) \setminus (E_x \cup E_y))
    \end{equation*}
    has sensitivity $0$, as for these edges we have $\{u,v\} \cap \{x,y\} = \emptyset$.
    
    For the \textbf{second part}, we will show that the sequence
    \begin{equation} \label{eq:sequence2}
        (|N(x) \triangle N(v)| + \cE_{x,v} - \beta \cdot \max(d(x),d(v)) : (x,v) \in E_x)
    \end{equation}
    is $(1.2 \epsagr, 1.2 \deltagr)$-DP.
    To that end, we use the Advanced Composition Theorem (\cref{thm:advanced_composition}).
    Let $k = |E_x|$; we have $k \le d(x)$
    (note that $d$ is the degree function of the input graph, which might be the smaller of $G$, $G'$, but we also have $(x,y) \not \in E_x$).
    Thus the sequence~\eqref{eq:sequence2} can be seen as a composition of $k$ functions (each with independent randomness),
    each of which is a sum of a function $|N(x) \triangle N(v)| - \beta \cdot \max(d(x),d(v))$, which has sensitivity at most $1 + \beta$,
    and Laplace noise $\cE_{x,v}$, which has magnitude at least
    \[
        \max \left( 1, \frac{\gamma \sqrt{k \cdot \ln (1/\deltagr)}}{\epsagr} \right)
    \]
    (where we used $\max(d(x),d(v)) \ge d(x) \ge k$).
    Define $\epsilon_x$ to be the inverse of this value, i.e.,
    \[
    \epsilon_x = \min \left( 1, \frac{\epsagr}{\gamma \sqrt{k \cdot \ln(1/\deltagr)}} \right) \,.
    \]
    Thus by \cref{theorem:mechanism} each of the $k$ functions is $((1 + \beta) \epsilon_x,0)$-DP.
    By \cref{thm:advanced_composition}, sequence \eqref{eq:sequence2} is $((1 + \beta) \epsilon', (1+\beta) \deltagr)$-DP,
    where
    \begin{align*}
        \epsilon' &= \sqrt{2 k \ln(1/\deltagr)} \epsilon_x + k \epsilon_x (e^{\epsilon_x} - 1) \\
        &\le \frac{\sqrt{2} \cdot \epsagr}{\gamma} + 2 k \epsilon_x^2 \\
        &\le \frac{\sqrt{2} \cdot \epsagr}{\gamma} + 2 k \frac{\epsagr^2}{\gamma^2 k \ln(1/\deltagr)} \\
        &= \epsagr \,,
    \end{align*}
    where for the first inequality we used that $e^{\epsilon_x} - 1 \le 2 \epsilon_x$ for $\epsilon_x \in [0,1]$,
    and the last equality follows by substituting the value of $\gamma$ (see \cref{definition:noised-agreement}) and reducing;
    our setting of $\gamma$ is in fact obtained by solving the quadratic equation $\frac{\sqrt{2} \cdot \epsagr}{\gamma} + \frac{2 \epsagr^2}{\gamma^2 \ln(1/\deltagr)} = \epsagr$.
    
    Finally, we have $1 + \beta \le 1.2$, and thus the sequence~\eqref{eq:sequence2} is $(1.2 \epsagr, 1.2 \deltagr)$-DP.

    The \textbf{third part} is analogous to the second.
    
    For the \textbf{fourth part},
    the sensitivity of the function
    $|N(x) \triangle N(y)| - \beta \cdot \max(d(x),d(y))$
    is $2 + \beta$
    (when edge $(x,y)$ is added, $x$ and $y$ disappear from $N(x) \triangle N(y)$, and $\max(d(x),d(y))$ increases by $1$).
    The Laplace noise $\cE_{x,y}$ yields $(\epsilon^*,0)$-differential privacy with
    \[
    \epsilon^* \le \frac{(2 + \beta) \epsagr}{\gamma \sqrt{\max(5,d(x),d(y)) \cdot \ln(1/\deltagr)}} \le \frac{2.2 \epsagr}{\sqrt{2} \cdot \sqrt{5 \cdot \ln(10)}} < 0.5 \epsagr \,,
    \]
    where we used that $\beta \le 0.2$, $\gamma \ge \sqrt{2}$ and $\deltagr \le \frac{0.5}{9.6} < 0.1$.

    Finally, we have a composition of four mechanisms with respective guarantees
    $(0,0)$, $(1.2 \epsagr, 1.2 \deltagr)$ (twice), and $(0.5 \epsagr, 0)$,
    and we can conclude the proof of the lemma using \cref{lem:composition}.
\end{proof}
Next we turn our attention to \cref{line:light}.
\begin{lemma}
    \label{lem:privacy_of_light}
    Under fixed
    noised-agreement status of all edges in $E(G) \cup E(G')$,
    consider \cref{line:light} as a randomized algorithm that,
    given $G$ or $G'$,
    outputs the heavy/light status of all vertices.
    It is $(\epsilon/4,0)$-DP.
\end{lemma}
\begin{proof}
    Under fixed
    noised-agreement status of all edges in $E(G) \cup E(G')$,
    the function $(l(v) - \lambda \cdot d(v) : v \in V)$ has sensitivity at most $2 \max(\lambda, 1-\lambda) \le 2$
    (when edge $(x,y)$ is added, the degrees $d(x)$ and $d(y)$ increase by $1$, and $l(x)$, $l(y)$ possibly increase by $1$ if $x$ and $y$ are not in noised agreement).
    Therefore the sequence $(\hl(v) - \lambda \cdot d(v) : v \in V)$ is $(\epsilon/4,0)$-DP by \cref{theorem:mechanism}, and it determines the heavy/light status of all vertices.
\end{proof}

Now we analyze the last line (\cref{line:component}).
\begin{restatable}{theorem}{laststep}
\label{thm:last_step}
By a \emph{state} let us denote the noised-agreement status of all edges in $E(G) \cup E(G')$ and heavy/light status of all vertices.
Under a fixed state,
consider \cref{line:component} as a \emph{deterministic} algorithm that,
given $G$ or $G'$,
outputs the final clustering.
Then
this clustering does not depend on whether the input graph is $G$ or $G'$,
except on a set of states that arises with probability at most $\frac34 \delta$ (when steps before \cref{line:component} are executed on either of $G$ or $G'$).
\end{restatable}

\cref{sec:proof_of_last_step} below is devoted to the proof of \cref{thm:last_step}.
Here, to get a flavor of the arguments, let us showcase what happens in the case when $x$ and $y$ are both heavy.
If they are not in noised agreement, we are done;
otherwise, the edge $(x,y)$ impacts the final solution
only if $x$ and $y$ are not otherwise connected in the sparsified graph $\tG$.
However, they are in fact likely to have many common neighbors in $\tG$,
as they are in noised agreement and both heavy.
Indeed, if all noise were zero, this would mean that $N(x) \cap N(y)$ is a large fraction of $\max(d(x),d(y))$
and that $x$ and $y$ do not lose many neighbors in \cref{line:noised_agreement}.
We show that with probability $1 - O(\delta)$,
all relevant noise is below a small fraction of $\max(d(x),d(y))$,
and the same argument still applies.
To get this, it is crucial that the agreement noise $\cE_{u,v}$
is of magnitude only roughly $\sqrt{\max(d(u),d(v)}$
(as opposed to e.g. $\max(d(u),d(v)$).

Once we have \cref{thm:last_step}, we can conclude:
\begin{theorem}
    \label{thm:privacy_main}
    \cref{alg:main} is $(\epsilon,\delta)$-DP.
\end{theorem}
\begin{proof}
    We repeatedly invoke \cref{lem:composition}
    to argue that the part of \cref{alg:main} consisting of Lines~\ref{line:noised_degree}--\ref{line:light}
    (that outputs noised-agreement and heavy/light statuses)
    is $(\epsilon/4,0) + (2.9 \epsagr, 2.4 \deltagr) + (\epsilon/4,0) = (\epsilon, \delta/4)$-DP
    by \cref{lem:privacy_of_noised_degrees,lem:privacy_of_agreement,lem:privacy_of_light}
    (recall the setting of $\epsagr$ and $\deltagr$ in \cref{definition:noised-agreement}).
    To conclude the proof,
    we argue about the last step using \cref{lem:low_prob_bad} and
    \cref{thm:last_step},
    which incurs a privacy loss of $(0,\frac34 \delta)$.
\end{proof}

\subsection{Proof of \cref{thm:last_step}} \label{sec:proof_of_last_step}

Let us analyze how adding a single edge $(x,y)$ can influence the output of \cref{line:component}.
Namely, we will show that it cannot, unless at least one of certain bad events happens.
We will list a collection of these bad events,
and then we will upper-bound their probability.

First, \textbf{if $x$ and $y$ are not in noised agreement}, then $(x,y)$ was removed in \cref{line:noised_agreement} and the two outputs will be the same.
In the remainder we assume that $x$ and $y$ are in noised agreement.
Similarly, we can assume that $x, y \in H$ (otherwise they cannot be in noised agreement).
    
\textbf{If $x$ and $y$ are both light}, then similarly $(x,y)$ will be removed in \cref{line:component}
    and the two outputs will be the same.
    
\textbf{If $x$ and $y$ are both heavy},
    then $(x,y)$ will survive in $\tG$.
    It will affect the output if and only if it connects two components that would otherwise not be connected.
    However, intuitively this is unlikely, because $x$ and $y$ are heavy and in noised agreement and thus they should have common neighbors in $\tG$.
    Below (\cref{lem:heavy_heavy}) we will show that if no bad events (also defined below) happen,
    then $x$ and $y$ indeed have common neighbors in $\tG$.
    
\textbf{If $x$ is heavy and $y$ is light},
    then similarly
    $(x,y)$ will survive in $\tG$,
    and it will affect the output if and only if it connects two components that would otherwise not be connected and that each contain a heavy vertex.
    More concretely, we claim that
    if the outputs are not equal,
    then $y$ must have a heavy neighbor $z \ne x$ (in $\tG$) that has no common neighbors with $x$ (except possibly $y$).
    For otherwise:
    \begin{itemize}
        \item if $y$ has a heavy neighbor $z \ne x$ that does have a common neighbor with $x$ (that is not $y$), then $x$ and $y$ are in the same component in $\tG$ regardless of the presence of $(x,y)$,
        \item if $y$ has no heavy neighbor except $x$,
        then (as light-light edges are removed) $y$ only has at most $x$ as a neighbor and therefore $(x,y)$ does not influence the output.
    \end{itemize}
    Let us call such a neighbor $z$ a \emph{bad neighbor}.
    Below (\cref{lem:heavy_light}) we will show that if no bad events (also defined below) happen,
    then $y$ has no bad neighbors.

Finally, \textbf{if $x$ is light and $y$ is heavy}:
analogous to the previous point.
We will require that $x$ have no bad neighbor,
i.e., neighbor $z \ne y$ that has no common neighbors with $y$.

\paragraph{Bad events.}
We start with two helpful definitions.

\begin{definition}
    We say that a vertex $v$ is \emph{\tvlight} (Truly Very light) if $l(v) \ge (\lambda + \lambda') d(v)$,
    i.e., $v$ lost a $(\lambda+\lambda')$-fraction of its neighbors in \cref{line:noised_agreement}.
\end{definition}

\begin{definition}
    We say that two vertices $u$, $v$ \tvdis (Truly Very disagree) if $|N(u) \triangle N(v)| \ge (\beta + \beta') \max(d(u), d(v))$.
\end{definition}

Recall from \cref{sec:algorithm} that we can set $\lambda' = \beta' = 0.1$.

Our bad events are the following:
\begin{enumerate}
    \item $x$ and $y$ \tvdis but are in noised agreement,
    \item $x$ is \tvlight but is heavy,
    \item the same for $y$,
    \item $x \in H$ but $d(x) < T_1$,
    \item the same for $y$,
    \item for each $z \in N(y) \setminus \{x,y\}$:
        \begin{enumerate}
            \item[6a.] $y$ and $z$ do not \tvdis, and $z$ is \tvlight but is heavy, (or)
            \item[6b.] $y$ and $z$ \tvdis, but are in noised agreement.
        \end{enumerate}
    \item similarly for each $z \in N(x) \setminus \{x,y\}$.
\end{enumerate}
Recall that we can assume that $x,y \in H$, so if bad event 4 does not happen, we have
\begin{equation}
    \label{eq:high_degree}
    d(x) \ge T_1
\end{equation}
and similarly for $y$ and bad event 5.

\paragraph{Heavy--heavy case.}
Let us denote the neighbors of a vertex $v$ in $\tG$ by $\tN(v)$;
also here we adopt the convention that $v \in \tN(v)$.
\begin{lemma}
    \label{lem:heavy_heavy}
    If $x$ and $y$ are heavy and bad events 1--5 do not happen,
    then $|\tN(x) \cap \tN(y)| \ge 3$,
    i.e., $x$ and $y$ have another common neighbor in $\tG$.
\end{lemma}
\begin{proof}
    Recall that we can assume that $x$ and $y$ are in noised agreement (otherwise the two outputs are equal).
    Since bad event 1 does not happen, $x$ and $y$ do not \tvdis, i.e.,
    \[|N(x) \triangle N(y)| < (\beta + \beta') \max(d(x), d(y)) \,. \]
    From this we get $\min(d(x),d(y)) \ge (1 - \beta - \beta') \max(d(x),d(y))$
    and thus $d(x) + d(y) = \min(d(x),d(y)) + \max(d(x),d(y)) \ge (2 - \beta - \beta') \max(d(x),d(y))$ and so
    \[
    |N(x) \triangle N(y)| < \frac{\beta + \beta'}{2 - \beta - \beta'} (d(x) + d(y)) \,.
    \]
    Since $x$ is heavy but bad event 2 does not happen,
    $x$ is not \tvlight,
    i.e.,
    $l(x) < (\lambda + \lambda') d(x)$.
    Moreover, $l(x) = |N(x) \setminus \tN(x)|$
    because $x$ is heavy (so there are no light-light edges incident to it).
    We use bad event 3 similarly for $y$.

    We will use the following property of any two sets $A$, $B$:
    \[
    |A \cap B| = \frac{|A| + |B| - |A \triangle B|}{2} \,.
    \]
    Taking these together, we have
    \begin{align*}
        |\tN(x) \cap \tN(y)| &\ge |N(x) \cap N(y)| - |N(x) \setminus \tN(x)| - |N(y) \setminus \tN(y)| \\
        &= \frac{d(x) + d(y) - |N(x) \triangle N(y)|}{2} - l(x) - l(y) \\
        &\ge \frac{1 - \beta - \beta'}{2 - \beta - \beta'} (d(x) + d(y)) - (\lambda + \lambda')(d(x) + d(y)) \\
        &= \left( \frac{1 - \beta - \beta'}{2 - \beta - \beta'} - \lambda - \lambda' \right) (d(x) + d(y)) \\
        &\ge 3 \,,
    \end{align*}
    where the last inequality follows
    since \[ \frac{1 - \beta - \beta'}{2 - \beta - \beta'} - \lambda - \lambda' \ge \frac{1 - 0.2 - 0.1}{2} - 0.2 - 0.1 = 0.05 > 0 \]
    and
    as, by \eqref{eq:high_degree}, we have $d(x) + d(y) \ge 2T_1$,
    and $T_1$ is large enough:
    \begin{equation} \label{eq:T1_large_enough_1}
        T_1 \ge \frac{1.5}{\frac{1 - \beta - \beta'}{2 - \beta - \beta'} - \lambda - \lambda'} \,.
    \end{equation}
\end{proof}

\paragraph{Heavy--light case.}
Without loss of generality assume that $x$ is heavy and $y$ is light.
Recall that a bad neighbor of $y$ is a vertex $z \in \tN(y) \setminus \{x,y\}$ that is heavy and has no common neighbors with $x$ (except possibly $y$).
\begin{lemma}
    \label{lem:heavy_light}
    If $x$ is heavy, $y$ is light, and
    bad events do not happen,
    then $y$ has no bad neighbors.
\end{lemma}
\begin{proof}
    Suppose
    that a vertex $z \in \tN(y) \setminus \{x,y\}$ is heavy;
    we will show that $z$ must have common neighbors with $x$.
    
    Since $z \in \tN(y)$, we have that $y$ and $z$ must be in noised agreement (otherwise $(y,z)$ would have been removed).
    Since bad event 6b does not happen, $y$ and $z$ do not \tvdis,
    i.e.,
    \[ |N(y) \triangle N(z)| < (\beta + \beta') \max(d(y), d(z)) \]
    which also implies that $d(z) \ge (1 - \beta - \beta') d(y)$.
    
    Since bad event 6a does not happen, and $y$ and $z$ do not \tvdis, and $z$ is heavy, thus $z$ is not \tvlight,
    i.e.,
    $l(z) < (\lambda + \lambda')d(z)$.
    
    As in the proof of \cref{lem:heavy_heavy},
    since bad events 1 and 2 do not happen, we have
    \[|N(x) \triangle N(y)| < (\beta + \beta') \max(d(x), d(y)) \,, \]
    which also implies that $d(x) \ge (1 - \beta - \beta') d(y)$
    and $l(x) < (\lambda + \lambda') d(x)$.
    Similarly as in that proof, we write
    \begin{align*}
        |\tN(x) \cap \tN(z)| &\ge |N(x) \cap N(z)| - |N(x) \setminus \tN(x)| - |N(z) \setminus \tN(z)| \\
        &= \frac{d(x) + d(z) - |N(x) \triangle N(z)|}{2} - l(x) - l(z) \\
        &\ge \frac{d(x) + d(z) - |N(x) \triangle N(y)| - |N(y) \triangle N(z)|}{2} - l(x) - l(z) \\
        &\ge \frac{d(x) + d(z) - (\beta + \beta') (d(x) + d(z))}{2} - (\lambda + \lambda') (d(x) + d(z)) \\
        &= \left( 1 - \beta - \beta' - 2(\lambda + \lambda') \right) \frac{d(x) + d(z)}{2} \\
        &\ge \left( 1 - \beta - \beta' - 2(\lambda + \lambda') \right) \frac{d(x) + (1 - \beta - \beta')d(y)}{2} \\
        &\ge \left( 1 - \beta - \beta' - 2(\lambda + \lambda') \right) \frac{2 - \beta - \beta'}{2} T_1 \\
        &\ge 2 \,,
    \end{align*}
    where the second-last inequality follows as, by \eqref{eq:high_degree}, we have $d(x), d(y) \ge T_1$,
    and the last inequality follows because
    \[ 1 - \beta - \beta' - 2(\lambda + \lambda') \ge 1 - 0.2 - 0.1 - 2\cdot(0.2 + 0.1) \ge 0.1 > 0 \]
    and $T_1$ is large enough:
    \begin{equation} \label{eq:T1_large_enough_2}
        T_1 \ge \frac{2 \cdot 2}{\left( 1 - \beta - \beta' - 2(\lambda + \lambda') \right) (2 - \beta - \beta')} \,.
    \end{equation}
\end{proof}

\paragraph{Bounding the probability of bad events.}
Roughly, our strategy is to union-bound over all the bad events.

\begin{fact}
    \label{fact:degree_enough}
    Let $A, c, d \ge 0$.
    If $d \ge \frac{\ln\left( \frac{c/2}{\delta} \right)}{A}$,
    then $\frac12 \exp(-A \cdot d) \le \frac{\delta}{c}$.
\end{fact}
\begin{proof}
    A straightforward calculation.
\end{proof}

\begin{claim} \label{claim:bad_event_1}
    The probability of bad event 1,
    conditioned on bad events 4 and 5 not happening,
    is at most $\delta/8$.
\end{claim}
\begin{proof}
    Start by recalling that by \eqref{eq:high_degree}, $d(x), d(y) \ge T_1$.
    We have that the sought probability is at most
    \begin{align*}
        \prob{\cE_{x,y} < - \beta' \cdot \max(d(x), d(y))}
        \le \frac12 \exp \left( - \frac{\beta' \cdot \max(d(x), d(y))}{\cE} \right)
    \end{align*}
    where we use $\cE$ to denote the magnitude of $\cE_{x,y}$, i.e.,
    \[\cE =  \max \left( 1, \frac{\gamma \sqrt{\max(d(x),d(y)) \cdot \ln (1/\deltagr)}}{\epsagr} \right)  \,. \]
    We will satisfy both
    \[ \frac12 \exp\left(- \beta' \cdot \max(d(x),d(y))\right) \le \frac{\delta}{8} \]
    and
    \[ \frac12 \exp\left(-\frac{\epsagr \cdot \beta' \cdot \max(d(x),d(y))}{\gamma \sqrt{\max(d(x),d(y)) \cdot \ln (1/\deltagr)}} \right)  \le \frac{\delta}{8}  \,. \]
    For the former, by applying \cref{fact:degree_enough} (for $c=8$, $A=\beta'$ and $d=\max(d(x),d(y))$) we get that it is enough to have
    $\max(d(x),d(y)) \ge \frac{\ln(4/\delta)}{\beta'}$,
    which holds when $T_1$ is large enough:
    \begin{equation} \label{eq:T1_large_enough_3}
        T_1 \ge \frac{\ln(4/\delta)}{\beta'} \,.
    \end{equation}
    For the latter, we want to satisfy
    \[ \frac12 \exp\left(-\frac{\epsagr \cdot \beta' \cdot \sqrt{\max(d(x),d(y))}}{\gamma \sqrt{\ln (1/\deltagr)}} \right)  \le \frac{\delta}{8}  \,. \]
    Use \cref{fact:degree_enough} (for $c=8$, $A=\frac{\epsagr \cdot \beta'}{\gamma \sqrt{\ln (1/\deltagr)}}$ and $d=\sqrt{\max(d(x),d(y))}$) to get that it is enough to have
    \[
    \sqrt{\max(d(x), d(y))} \ge \frac{\ln(4/\delta) \cdot \gamma \cdot \sqrt{\ln(1/\deltagr)}}{\epsagr \cdot \beta'} \,,
    \]
    which is true when $T_1$ is large enough:
    \begin{equation} \label{eq:T1_large_enough_4}
        T_1 \ge \left( \frac{\ln(4/\delta) \cdot \gamma}{\epsagr \cdot \beta'} \right)^2 \cdot \ln(1/\deltagr) \,.
    \end{equation}
\end{proof}

\begin{claim} \label{claim:bad_event_2}
    The probability of bad event 2,
    conditioned on bad events 4 and 5 not happening,
    is at most $\delta/32$.
\end{claim}
\begin{proof}
    Start by recalling that by \eqref{eq:high_degree}, $d(x) \ge T_1$.
    If $x$ is TV-light but heavy, then we must have $Y_x < \lambda' \cdot d(x)$.
    We have that the sought probability is at most
    \[ \frac12 \exp \left( - \frac{\lambda' \cdot d(x) \cdot \epsilon}{8} \right)
    \]
    and by \cref{fact:degree_enough}
    (with $c=32$, $d=d(x)$ and $A = \frac{\lambda' \cdot \epsilon}{8}$)
    this is at most $\delta/32$ because $d(x) \ge T_1$ and $T_1$ is large enough:
    \begin{equation} \label{eq:T1_large_enough_5}
        T_1 \ge \frac{8 \ln(16/\delta)}{\lambda' \cdot \epsilon} \,.
    \end{equation}
\end{proof}

\begin{claim} \label{claim:bad_event_4}
    The probability of bad event 4 is at most $\delta/32$. 
\end{claim}
\begin{proof}
    For bad event 4 to happen, we must have
    $Z_x \ge T_0 - T_1 = \frac{8 \ln(16/\delta)}{\epsilon}$;
    as $Z_x \sim \Lap(8/\epsilon)$,
    this happens with probability $\frac12 \exp(-\ln(16/\delta)) = \delta/32$.
\end{proof}
The following two facts are more involved versions of of \cref{fact:degree_enough}.
\begin{fact}
    \label{fact:lambert_function}
    Let $A, d \ge 0$.
    If $d \ge \frac{1.6 \ln\left(\frac{4}{\delta A}\right)}{A}$,
    then $\frac12 \exp(-A \cdot d) \le \frac{\delta}{8 d}$.
\end{fact}
\begin{proof}
    We use the following analytic inequality:
    for $\alpha, x > 0$, if $x \ge 1.6 \ln(\alpha)$,
    then $x \ge \ln(\alpha x)$.
    We substitute $x = A \cdot d$ and $\alpha = \frac{4}{\delta A}$.
    Then by the analytic inequality,
    $A \cdot d \ge \ln\left(\frac{4 d}{\delta}\right)$.
    Negate and then exponentiate both sides.
\end{proof}

\begin{fact}
    \label{fact:lambert_function_2}
    Let $A, d \ge 0$.
    If $\sqrt{d} \ge \frac{2.8 \cdot \left( 1 + \ln\left( \frac{2}{\sqrt{\delta} A} \right) \right)}{A}$,
    then $\frac12 \exp(-A \cdot \sqrt{d}) \le \frac{\delta}{8 d}$.
\end{fact}
\begin{proof}
    We use the following analytic inequality:
    for $\alpha, x > 0$, if $x \ge 2.8 (\ln(\alpha) + 1)$,
    then $x \ge 2 \ln(\alpha x)$.
    We substitute $x = A \sqrt{d}$ and $\alpha = \frac{2}{\sqrt{\delta} A}$.
    Then by the analytic inequality,
    $A \cdot \sqrt{d} \ge \ln\left(\frac{4 d}{\delta}\right)$.
    Negate and then exponentiate both sides.
\end{proof}

\begin{claim} \label{claim:bad_event_6a}
    For any $z \in N(y) \setminus \{x,y\}$, the probability of bad event 6a for $z$,
    conditioned on bad events 4 and 5 not happening,
    is at most $\frac{\delta}{8 d(y)}$.
\end{claim}
\begin{proof}
    The proof is similar as for \cref{claim:bad_event_2} but somewhat more involved as $d(y)$ appears also in the probability bound.

    When $z$ is \tvlight but heavy, we must have $Y_z < - \lambda' \cdot d(z)$.
    When $y$ and $z$ do not \tvdis, we have $d(z) \ge (1 - \beta - \beta') d(y)$.
    Thus, if bad event 6a happens, we must have
    $Y_z < -\lambda' \cdot (1 - \beta - \beta') d(y)$.
    Thus the sought probability is at most
    \begin{align*}
        \prob{Y_z < - \lambda' \cdot (1 - \beta - \beta') d(y)}
        = \frac12 \exp\left(- \frac{\lambda' \cdot (1 - \beta - \beta') d(y) \cdot \epsilon}{8}\right) \,.
    \end{align*}
    By \cref{fact:lambert_function} (invoked for $d = d(y)$ and $A = \frac{\lambda' \cdot (1 - \beta - \beta') \cdot \epsilon}{8}$), this is at most $\frac{\delta}{8 d(y)}$
    because $d(y) \ge T_1$ by \eqref{eq:high_degree}
    and $T_1$ is large enough:
    \begin{equation} \label{eq:T1_large_enough_6}
        T_1 \ge \frac{1.6 \ln\left(\frac{4 \cdot 8}{\delta \lambda' \cdot (1 - \beta - \beta') \cdot \epsilon}\right) \cdot 8}{\lambda' \cdot (1 - \beta - \beta') \cdot \epsilon} \,.
    \end{equation}
\end{proof}

\begin{claim} \label{claim:bad_event_6b}
    For any $z \in N(y) \setminus \{x,y\}$, the probability of bad event 6b for $z$,
    conditioned on bad events 4 and 5 not happening,
    is at most $\frac{\delta}{8 d(y)}$.
\end{claim}
\begin{proof}
    The proof is similar as for \cref{claim:bad_event_1}
    but somewhat more involved as $d(y)$ appears also in the probability bound.
    Start by recalling that by \eqref{eq:high_degree}, $d(y) \ge T_1$.
    We have that the sought probability is at most
    \begin{align*}
        \prob{\cE_{y,z} < - \beta' \cdot \max(d(y), d(z))}
        \le \frac12 \exp \left( - \frac{\beta' \cdot \max(d(y), d(z))}{\cE} \right)
    \end{align*}
    where we use $\cE$ to denote the magnitude of $\cE_{y,z}$, i.e.,
    \[\cE =  \max \left( 1, \frac{\gamma \sqrt{\max(d(y),d(z)) \cdot \ln (1/\deltagr)}}{\epsagr} \right)  \,. \]
    We will satisfy both
    \begin{equation}
        \label{eq:sdfertfsdf}
        \frac12 \exp\left(- \beta' \cdot \max(d(y),d(z))\right) \le \frac12 \exp\left(- \beta' \cdot d(y)\right) \le \frac{\delta}{8 d(y)}
    \end{equation}
    and
    \begin{equation}
        \label{eq:dfkjjifh}
        \frac12 \exp\left(-\frac{\epsagr \cdot \beta' \cdot \max(d(y),d(z))}{\gamma \sqrt{\max(d(y),d(z)) \cdot \ln (1/\deltagr)}} \right)  \le
    \frac12 \exp\left(-\frac{\epsagr \cdot \beta' \cdot \sqrt{d(y)}}{\gamma \sqrt{\ln (1/\deltagr)}} \right)  \le
    \frac{\delta}{8 d(y)}  \,. 
    \end{equation}
    For the former, by applying \cref{fact:lambert_function} (for $A=\beta'$ and $d=d(y)$) we get that \eqref{eq:sdfertfsdf} holds because $d(y) \ge T_1$ and $T_1$ is large enough:
    \begin{equation} \label{eq:T1_large_enough_7}
        T_1 \ge \frac{1.6 \ln\left(\frac{4}{\delta \cdot \beta'}\right)}{\beta'} \,.
    \end{equation}
    For the latter, by applying \cref{fact:lambert_function_2} (for $A = \frac{\epsagr \cdot \beta'}{\gamma \sqrt{\ln (1/\deltagr)}}$ and $d=d(y)$) we get that \eqref{eq:dfkjjifh} holds because $d(y) \ge T_1$ and $T_1$ is large enough:
    \begin{equation} \label{eq:T1_large_enough_8}
        T_1
        \ge \left( \frac{2.8 \left( 1 + \ln \left( \frac{2}{\sqrt{\delta} A} \right) \right)}{A} \right)^2
        =
        \left( \frac{2.8 \left( 1 + \ln \left( \frac{2 \gamma \sqrt{\ln (1/\deltagr)}}{\sqrt{\delta} \epsagr \cdot \beta'} \right) \right) \gamma \sqrt{\ln (1/\deltagr)}}{\epsagr \cdot \beta'} \right)^2 \,.
    \end{equation}
\end{proof}

Now we may conclude the proof of \cref{thm:last_step}.
We use the property that if $A$, $B$ are events, then
$\prob{A \cup B} \le \prob{A} + \prob{B \mid \text{not $A$}}$
(with $A$ being bad events 4 or 5).
By \cref{claim:bad_event_4},
the probability of bad events 4 or 5 is at most $\delta/16$.
Conditioned on these not happening,
bad event 1 is handled by \cref{claim:bad_event_1}
and bad events 2--3 are handled by \cref{claim:bad_event_2};
these incur $\delta/8 + 2 \cdot \delta/32$,
in total $\delta/4$ so far.
Next, there are $d(y)$ bad events of type 6a (and the same for 6b),
thus we get $2 \cdot d(y) \cdot \frac{\delta}{8 d(y)} = \delta/4$
by \cref{claim:bad_event_6a,claim:bad_event_6b};
and we get the same from bad events 7a and 7b.
Summing everything up yields $\frac34 \delta$.
\hfill\ensuremath{\blacksquare}

\section{Analysis of Approximation} \label{sec:approximation_analysis}

For vectors $\myvec{\beta} \in \bbR_{\ge 0}^{V \times V}$ and $\myvec{\lambda} \in \bbR_{\ge 0}^V$, let $\AlgCC(\myvec{\beta}, \myvec{\lambda})$ be the algorithm from \cite{cohen2021correlation} that uses $\myvec{\beta}_{u,v}$ to decide an agreement between $u$ and $v$ and uses $\myvec{\lambda}_v$ to decide whether $v$ is light or heavy.
Let $\AlgCC(\myvec{\beta}, \myvec{\lambda}, \Erem)$ (stated as \cref{alg:AlgCC}) be a variant of $\AlgCC(\myvec{\beta}, \myvec{\lambda})$
that at the very first step removes $\Erem$, then executes the remaining steps, and finally (as in \cref{alg:main}) outputs light vertices as singleton clusters.
\begin{algorithm2e}[h]
\LinesNumbered
  \Input{$G=(V,E)$: a graph \\
		$\myvec{\beta} \in \bbR_{\ge 0}^{V \times V}$ : agreement parameter \\
		$\myvec{\lambda} \in \bbR_{\ge 0}^V$ : threshold for light vertices \\
		$\Erem$ : a subset of edges to be removed}

	Remove the edges in $\Erem$.

	Discard from $G$ the edges that are not in agreement where $u$ and $v$ are in the agreement if $|N(u)\triangle N(v)| < \myvec{\beta}_{u, v} \cdot \max(d(u), d(v))$. (First compute the set of these edges. Then remove this set.)
   
    Let $l(v)$ be the number of edges incident to $v$ discarded in the previous steps. Call a vertex $v$ \emph{light} if $l(v)> \myvec{\lambda}_v d(v)$, and  otherwise call $v$ \emph{heavy}.

    Discard all edges whose both endpoints are light.
    Call the current graph  $\hG$, or the \emph{sparsified graph}. Compute its connected components.
    Output  the heavy vertices in each component $C$ as a cluster. Each light vertex  is output as a singleton.
    \label{line:AlgCC-light-singleton}

  \caption{$\AlgCC(\myvec{\beta}, \myvec{\lambda}, \Erem)$, used for the approximation analysis.}
	\label{alg:AlgCC}
\end{algorithm2e}

The strategy of our proof is to map the behavior of \cref{alg:main} to $\AlgCC(\myvec{\beta}, \myvec{\lambda}, \Erem)$ for appropriately set $\myvec{\beta}, \myvec{\lambda}$, and $\Erem$.
We remark that $\AlgCC$ is not actually executed by our DP-approach, but it is rather a hypothetical algorithm which (when appropriately instantiated) resembles approximation guarantees of \cref{alg:main}. Moreover, \AlgCC has similar structure to the approach of \cite{cohen2021correlation}, enabling us to reuse some of the results from that prior work to establish approximation guarantees of \AlgCC; certain steps of \AlgCC (such as the removal of $\Erem$) are analyzed independently of the prior work.

Given two vectors $\myvec{x}$ and $\myvec{y}$ labeled by a set $\cS$, we say that $\myvec{x} \le \myvec{y}$ iff $\myvec{x}_s \le \myvec{y}_s$ for each $s \in \cS$.

\begin{restatable}{lemma}{corollarysamecluster}
\label{corollary:same-cluster-different-clusters}
    Let $\myvec{\beta^L}, \myvec{\beta^U} \in \bbR_{\ge 0}^{V \times V}$ and $\myvec{\lambda^L}, \myvec{\lambda^U} \in \bbR_{\ge 0}^V$ such that $\myvec{\beta^U} \ge \myvec{\beta^L}$ and $\myvec{\lambda^U} \ge \myvec{\lambda^L}$.
    \begin{enumerate}[(i)]
        \item\label{item:u-and-v-same-cluster} If $u$ and $v$ are in the same cluster of $\AlgCC(\myvec{\beta^L}, \myvec{\lambda^L}, \Erem)$, then $u$ and $v$ are in the same cluster of $\AlgCC(\myvec{\beta^U}, \myvec{\lambda^U}, \Erem)$.
        \item\label{item:u-and-v-different-clusters} If $u$ and $v$ are in different clusters of $\AlgCC(\myvec{\beta^U}, \myvec{\lambda^U}, \Erem)$, then $u$ and $v$ are different clusters of $\AlgCC(\myvec{\beta^L}, \myvec{\lambda^L}, \Erem)$.
    \end{enumerate}
\end{restatable}
The proof of \cref{corollary:same-cluster-different-clusters} is given in \cref{sec:corollary:same-cluster-different-clusters}. We now derive the following claim that enables us to sandwich $\cost(\AlgCC(\myvec{\beta}, \myvec{\lambda}, \Erem))$ between two other instances of $\AlgCC$.

\begin{lemma}
\label{lemma:sandwich-beta-lambda}
    Let $\myvec{\beta^L}, \myvec{\beta}, \myvec{\beta^U} \in \bbR_{\ge 0}^{V \times V}$ and $\myvec{\lambda^L}, \myvec{\lambda}, \myvec{\lambda^U} \in \bbR_{\ge 0}^V$ such that $\myvec{\beta^U} \ge \myvec{\beta} \ge \myvec{\beta^L}$ and $\myvec{\lambda^U} \ge \myvec{\lambda} \ge \myvec{\lambda^L}$. Then 
    \[
        \cost(\AlgCC(\myvec{\beta}, \myvec{\lambda}, \Erem)) \le \cost(\AlgCC(\myvec{\beta^U}, \myvec{\lambda^U}, \Erem)) + \cost(\AlgCC(\myvec{\beta^L}, \myvec{\lambda^L}, \Erem)).
    \]
\end{lemma}
\begin{proof}
    We first upper-bound the cost of $\AlgCC(\myvec{\beta}, \myvec{\lambda}, \Erem)$ incurred by ``-'' edges. If a ``-'' edge $\{u, v\}$ adds to the cost of clustering, it is because $u$ and $v$ are in the same cluster. By \cref{corollary:same-cluster-different-clusters}~\eqref{item:u-and-v-same-cluster}, if $u$ and $v$ are in the same cluster of $\AlgCC(\myvec{\beta}, \myvec{\lambda}, \Erem)$, then they are in the same cluster of $\AlgCC(\myvec{\beta^U}, \myvec{\lambda^U}, \Erem)$ as well. Hence, the cost of $\AlgCC(\myvec{\beta}, \myvec{\lambda}, \Erem)$ incurred by ``-'' edges is upper-bounded by $\cost(\AlgCC(\myvec{\beta^U}, \myvec{\lambda^U}, \Erem))$.
    
    In a similar way, we upper-bound the cost of $\AlgCC(\myvec{\beta}, \myvec{\lambda}, \Erem)$ incurred by ``+'' edges by \\ $\cost(\AlgCC(\myvec{\beta^L}, \myvec{\lambda^L}, \Erem))$. If a ``+'' edge $\{u, v\}$ adds to the cost of clustering, it is because $u$ and $v$ are in different clusters. By \cref{corollary:same-cluster-different-clusters}~\eqref{item:u-and-v-different-clusters}, if $u$ and $v$ are in different cluster of $\AlgCC(\myvec{\beta}, \myvec{\lambda}, \Erem)$, then they are in different clusters of $\AlgCC(\myvec{\beta^L}, \myvec{\lambda^L}, \Erem)$ as well. Hence, the cost of the output of $\AlgCC(\myvec{\beta}, \myvec{\lambda}, \Erem)$ incurred by ``+'' edges is upper-bounded by $\cost(\AlgCC(\myvec{\beta^L}, \myvec{\lambda^L}, \Erem))$.
\end{proof}

We now analyze the effect of removing edges incident to vertices which are not in $H$ defined on \cref{line:noised_degree} of \cref{alg:main}. To simplify the analysis, we first ignore the step that outputs light vertices as singletons (\cref{line:light-to-singletons} of \cref{alg:main} and \cref{line:AlgCC-light-singleton} of \cref{alg:AlgCC}).
For a threshold $T \in \bbR_{\ge 0}$, let $E_{\le T}$ a \emph{subset} of edges incident to vertices of degree at most $T$.

\begin{lemma}
\label{lemma:additive-cost}
    Let $\AlgCC'$ be a version of $\AlgCC$ that does not make singletons of light vertices on \cref{line:AlgCC-light-singleton} of \cref{alg:AlgCC}.
    Let $\myvec{\beta} \in \bbR_{\ge 0}^{V \times V}$ and $\myvec{\lambda} \in \bbR_{\ge 0}^V$ be two constant vectors, i.e., $\myvec{\beta} = \beta \myvec{1}$ and $\myvec{\lambda} = \lambda \myvec{1}$. Assume that $5 \beta + 2\lambda < 1$.
    Then, it holds
    \[
        \cost(\AlgCC'(\myvec{\beta}, \myvec{\lambda}, E_{\le T})) \le O(OPT/(\beta \lambda))  + O(n \cdot T /(1-4\beta)^3)\,,
    \]
    where OPT denotes the cost of the optimum clustering for the input graph.
\end{lemma}
\begin{proof}
    Consider a non-singleton cluster $C$ output by $\AlgCC'(\myvec{\beta}, \myvec{\lambda}, \emptyset)$.
    Let $u$ be a vertex in $C$. We now show that for any $v \in C$, such that $u$ or $v$ is heavy, it holds that $d(v) \ge (1 - 4 \beta) d(u)$. To that end, we recall that in \cite[Lemma 3.3 of the arXiv version]{cohen2021correlation} was shown
    \begin{equation}\label{eq:4-weak-agreement}
        |N(u) \triangle N(v)| \le 4 \beta \max\{d(u), d(v)\}.
    \end{equation}
    Assume that $d(u) \ge d(v)$, as otherwise $d(v) \ge (1-4\beta) d(u)$ holds directly. Then, from \cref{eq:4-weak-agreement} we have
    \[
        d(u) - d(v) \le |N(u) \triangle N(v)| \le 4 \beta d(u),
    \]
    further implying
    \[
        d(v) \ge (1 - 4 \beta) d(u).
    \]
    Moreover, this provides a relation between $d(v)$ and $d(u)$ even if both vertices are light.
    To see that, fix any heavy vertex $z$ in the cluster. Any vertex $u$ has $d(u) \le d(z) / (1-4 \beta)$ and also $d(u) \ge (1-4 \beta) d(z)$. 
    This implies that if $u$ and $v$ belong to the same cluster than $d(u) \ge (1-4\beta)^2 d(v)$, even if both $u$ and $v$ are light.
    
    Let $E_{\le T}$ be a subset (any such) of edges incident to vertices with degree at most $T$.
    We will show that forcing $\AlgCC'$ to remove $E_{\le T}$ does not affect how vertices of degree at least $T / (1 - 4\beta)^3$ are clustered by $\AlgCC'$. To see that, observe that a vertex $x$ having degree at most $T$ and a vertex $y$ having degree at least $T/(1 - \beta) + 1$ are not in agreement. Hence, forcing $\AlgCC'$ to remove $E_{\le T}$ does not affect whether vertex $y$ is light or not.
    
    However, removing $E_{\le T}$ might affect whether a vertex $z$ with degree $T/(1-\beta) < T/(1 - 4\beta)$ is light or not.
    Nevertheless, from our discussion above, a vertex $y$ with degree at least $T/(1 - 4\beta)^3$ is not clustered together with $z$ by $\AlgCC'(\beta, \lambda, \emptyset)$, regardless of whether $z$ is heavy or light. 
    
    This implies that the cost of clustering vertices of degree at least $T / (1 - 4\beta)^3$ by $\AlgCC'(\beta, \lambda, E_{\le T})$ is upper-bounded by $\cost(\AlgCC'(\myvec{\beta}, \myvec{\lambda}, \emptyset)) \le O(OPT/(\beta \lambda))$. Notice that the inequality follows since $\AlgCC'(\myvec{\beta}, \myvec{\lambda}, \emptyset)$ is a $O(1/(\beta \lambda))$-approximation of $OPT$ and $\beta < 0.2$.
    
    It remains to account for the cost effect of $\AlgCC'(\myvec{\beta}, \myvec{\lambda}, E_{\le T})$ on the vertices of degree less than $T/(1 - 4 \beta)^3$. This part of the analysis follows from the fact that forcing $\AlgCC'$ to remove $E_{\le T}$ only reduces connectivity compared to the output of $\AlgCC'$ without removing $E_{\le T}$. 
    That is, in addition to removing edges even between vertices that might be in agreement, removal of $E_{\le T}$ increases a chance for a vertex to become light.
    Hence, the clusters of $\AlgCC'$ with removals of $E_{\le T}$ are only potentially further clustered compared to the output of $\AlgCC'$ without the removal. This means that $\AlgCC'$ with the removal of $E_{\le T}$ potentially cuts additional ``+'' edges, but it does not include additional ``-'' edges in the same cluster. Given that only vertices of degree at most $T/(1-4\beta)^3$ are affected, the number of additional ``+'' edges cut is $O(n \cdot T/(1-4\beta)^3)$.
    
    This completes the analysis.
\end{proof}

\begin{restatable}{lemma}{thmapproxproof}
\label{thm:approx-proof}
Let \cref{alg:main}' be a version of \cref{alg:main} that does not make singletons of light vertices on \cref{line:light-to-singletons}.
Assume that $5 \beta + 2\lambda < 1/1.1$ and also assume that $\beta$ and $\lambda$ are positive constants.
With probability at least $1 - n^{-2}$, \cref{alg:main}' provides a solution which has $O(1)$ multiplicative and $O\rb{n \cdot \rb{\tfrac{\log{n}}{\eps} + \tfrac{\log^2 n \cdot \log(1/\delta)}{\min(1, \eps^2)}}}$ additive approximation.
\end{restatable}
\begin{proof}
    We now analyze under which condition noised agreement and $\hl(v)$ can be seen as a slight perturbation of $\beta$ and $\lambda$. That will enable us to employ \cref{lemma:sandwich-beta-lambda,lemma:additive-cost} to conclude the proof of this theorem.
    \paragraph{Analyzing noised agreement.}
    Recall that a noised agreement (\cref{definition:noised-agreement}) states
    \[
        |N(u)\triangle N(v)|+\cE_{u,v} < \beta \cdot
\max(d(u), d(v)).
    \]
    This inequality can be rewritten as
    \[
        |N(u)\triangle N(v)| < \rb{1 - \frac{\cE_{u,v}}{\beta \cdot \max(d(u), d(v))}} \beta \cdot
\max(d(u), d(v)).
    \]
    As a reminder, $\cE_{u,v}$ is drawn from $\Lap(C_{u, v} \cdot \sqrt{\max(d(u), d(v)) \ln(1/\delta)} / \epsagr)$, where $C_{u, v}$ can be upper-bounded by $C = \sqrt{4 \epsagr + 1} + 1$.
    Let $b = C \cdot \sqrt{\max(d(u), d(v)) \ln(1/\delta)} / \epsagr$.
    From \cref{lem:folk} we have that
    \[
        \prob{|\cE_{u, v}| > 5 \cdot b \cdot \log{n}} \le n^{-5}.
    \]
    Therefore, with probability at least $1-n^{-5}$ we have that
    \[
        \left|\frac{\cE_{u,v}}{\beta \cdot \max(d(u), d(v))}\right| \le  \frac{5 \cdot \log n \cdot C \cdot \sqrt{\max(d(u), d(v)) \ln(1/\delta)} }{\epsagr \cdot \beta \cdot \max(d(u), d(v))} = \frac{5 \cdot \log n \cdot C \cdot \sqrt{\ln(1/\delta)}}{\epsagr \cdot \beta \cdot \sqrt{\max(d(u), d(v))}}
    \]
    Therefore, for $\max(d(u), d(v)) \ge \frac{2500 \cdot C^2 \cdot \log^2 n \cdot \log(1/\delta)}{\beta^2 \cdot \epsagr^2}$ we have that with probability at least $1-n^{-5}$ it holds
    \[
        1 - \frac{\cE_{u,v}}{\beta \cdot \max(d(u), d(v))} \in [9/10, 11/10].
    \]
    
    \paragraph{Analyzing noised $l(v)$.}
    As a reminder, $\hl(v) = l(v) + Y_v$, where $Y_v$ is drawn from $\Lap(8 / \eps)$. The condition $\hl(v) > \lambda d(v)$ can be rewritten as
    \[
        l(v) > \rb{1 - \frac{Y_v}{\lambda d(v)}} \lambda d(v).
    \]
    Also, we have
    \[
        \prob{|Y_v| > \frac{40 \log{n}}{\eps}} < n^{-5}.
    \]
    Hence, if $d(v) \ge \frac{400 \log{n}}{\lambda \eps}$ then with probability at least $1 -  n^{-5}$ we have that
    \[
        1 - \frac{Y_v}{\lambda d(v)} \in [9/10, 11/10].
    \]
    \paragraph{Analyzing noised degrees.}
    Recall that noised degree $\hd(v)$ is defined as $\hd(v) = d(v) + Z_v$, where $Z_v$ is drawn from $\Lap(8/\eps)$. From \cref{lem:folk} we have
    \[
        \prob{|Z_v| > \frac{40 \log{n}}{\eps}} < n^{-5}.
    \]
    Hence, with probability at least $1-n^{-5}$, a vertex of degree at least $T_0 + 40 \log{n} / \eps$ is in $H$ defined on \cref{line:noised_degree} of \cref{alg:main}. Also, with probability at least $1 - n^{-5}$ a vertex with degree less than $T_0 - 40 \log{n} / \eps$ is not in $H$.
    
    \paragraph{Combining the ingredients.}
    Define
    \[
        T' = \max\rb{\frac{400 \log{n}}{\lambda \eps}, \frac{2500 \cdot C^2 \cdot \log^2 n \cdot \log(1/\delta)}{\beta^2 \cdot \epsagr^2}}
    \]
    Our analysis shows that for a vertex $v$ such that $d(v) \ge T'$ the following holds with probability at least $1 - 2 n^{-5}$:
    \begin{enumerate}[(i)]
        \item\label{item:perturb-beta} The perturbation by $\cE_{u,v}$ in \cref{definition:noised-agreement} can be seen as multiplicatively perturbing $\myvec{\beta}_{u,v}$ by a number from the interval $[-1/10, 1/10]$.
        \item\label{item:perturb-lambda} The perturbation of $l(v)$ by $Y_v$ can be seen as multiplicatively perturbing $\myvec{\lambda}_v$ by a number from the interval $[-1/10, 1/10]$.
    \end{enumerate}
    Let $T = T_0 + \tfrac{40 \log n}{\eps}$.
    Let $T_0 \ge T' + \tfrac{40 \log n}{\eps}$. Note that this imposes a constraint on $T_1$, which is
    \begin{equation}\label{eq:constraint-on-T_1}
        T_1 \ge T' + \frac{40 \log n}{\eps} - \frac{8\log (16/\delta)}{\epsilon}.
    \end{equation}
    Then, following our analysis above, each vertex in $H$ has degree at least $T'$, and each vertex of degree at least $T$ is in $H$. Let $E_{\le T}$ be the set of edges incident to vertices which are not in $H$; these edges are effectively removed from the graph.
    Observe that for a vertex $u$ which do not belong to $H$ it is irrelevant what $\myvec{\beta}_{u, \cdot}$ values are or what $\myvec{\lambda}_u$ is, as all its incident edges are removed. 
    To conclude the proof, define $\myvec{\beta^L} = 0.9 \cdot \beta \cdot \myvec{1}$, $\myvec{\beta^U} = 1.1 \cdot \beta \cdot \myvec{1}$, $\myvec{\lambda^L} = 0.9 \cdot \lambda \cdot \myvec{1}$, and $\myvec{\lambda^U} = 1.1 \cdot \lambda \cdot \myvec{1}$. By \cref{lemma:sandwich-beta-lambda} and Properties~\ref{item:perturb-beta} and \ref{item:perturb-lambda} we have that
    \[
        \cost(\cref{alg:main}') \le \cost(\AlgCC(\myvec{\beta^L}, \myvec{\lambda^L}, E_{\le T})) + \cost(\AlgCC(\myvec{\beta^U}, \myvec{\lambda^U}, E_{\le T})).
    \]
    By \cref{lemma:additive-cost} the latter sum is upper-bounded by $O(OPT/(\beta \lambda))  + O(n \cdot T /(1-4\beta)^3)$. Note that we replace the condition $5 \beta + 2\lambda$ in the statement of \cref{lemma:additive-cost} by $5 \beta + 2\lambda < 1/1.1$ in this lemma so to account for the perturbations. Moreover, we can upper-bound $T$ by
    \[
        T \le O\rb{\frac{\log{n}}{\lambda \eps} + \frac{\log^2 n \cdot \log(1/\delta)}{\beta^2 \cdot \min(1, \eps^2)}}.
    \]
    In addition, all discussed bound hold across all events with probability at least $1 - n^{-2}$. This concludes the analysis.
\end{proof}
\cref{thm:approx-proof} does not take into account the cost incurred by making singleton-clusters from the light vertices, as performed on \cref{line:light-to-singletons} of \cref{alg:main}. The next claim upper-bounds that cost as well.

\begin{restatable}{lemma}{lemapprosingletons}
\label{lem:appro-singletons}
    Consider all lights vertices defined in \cref{line:light-to-singletons} of \cref{alg:main}. Assume that $5 \beta + 2\lambda < 1/1.1$. Then, with probability at least $1 - n^{-2}$, making as singleton clusters any subset of those light vertices increases the cost of clustering by $O(\mathrm{OPT}/(\beta \cdot \lambda)^2)$, where $\mathrm{OPT}$ denotes the cost of the optimum clustering for the input graph.
\end{restatable}
\begin{proof}
    Consider first a single light vertex $v$ which is not a singleton cluster.
    Let $C$ be the cluster of $\hG'$ that $v$ initially belongs to. We consider two cases. First, recall that from our proof of \cref{thm:approx-proof} that, with probability at least $1 - n^{-2}$, we have that $0.9 \lambda \le \myvec{\lambda}_v \le 1.1 \lambda$ and $0.9 \beta \le \myvec{\beta}_{u, v} \le 1.1 \beta$, where $\myvec{\lambda}$ and $\myvec{\beta}$ are inputs to \AlgCC.
    
    \paragraph{Case 1: $v$ has at least $\myvec{\lambda}_v/2$ fraction of neighbors outside $C$.}
        In this case, the cost of having $v$ in $C$ is already at least $d(v) \cdot \myvec{\lambda}_v/2 \ge d(v) \cdot 0.9 \cdot \lambda/2$, while having $v$ as a singleton has cost $d(v)$.
        
    \paragraph{Case 2: $v$ has less then $\myvec{\lambda}_v/2$ fraction of neighbors outside $C$.}
        Since $v$ is not in agreement with at least $\myvec{\lambda}_v$ fraction of its neighbors, this case implies that at least $\myvec{\lambda}_v/2 \ge 0.9 \cdot \lambda / 2$ fraction of those neighbors are in $C$. We now develop a charging arguments to derive the advertised approximation.
        
        Let $x \in C$ be a vertex that $v$ is not in a agreement with. Then, for a fixed $x$ and $v$ in \emph{the same} cluster of $\hG'$, there are at least $O(d(v) \beta)$ vertices $z$ (incident to $x$ or $v$, but not to the other vertex) that the current clustering is paying for. In other words, the current clustering is paying for edges of the form $\{z, x\}$ and $\{z, v\}$; as a remark, $z$ does not have to belong to $C$. Let $Z(v)$ denote the \emph{multiset} of all such edges for a given vertex $v$. We charge each edge in $Z(v)$ by $O(1/(\beta \lambda))$.
        
        On the other hand, making $v$ a singleton increases the cost of clustering by at most $d(v)$. We now want to argue that there is enough charging so that we can distribute the cost $d(v)$ (for making $v$ a singleton cluster) over $Z(v)$ and, moreover, do that for all light vertices $v$ simultaneously.
        There are at least $O(\beta \cdot d(v) \cdot \lambda \cdot d(v))$ edges in $Z(v)$; recall that $Z(v)$ is a multiset. We distribute uniformly the cost $d(v)$ (for making $v$ a singleton) across $Z(v)$, incurring $O(1 / (\beta \cdot \lambda \cdot d(v)))$ cost per an element of $Z(v)$.
        
        Now it remains to comment on how many times an edge appears in the union of all $Z(\cdot)$ multisets. Edge $z_e = \{x, y\}$ in included in $Z(\cdot)$ when $x$ and its neighbor, or $y$ and its neighbor are considered. Moreover, those neighbors belong to the same cluster of $\hG'$ and hence have similar degrees (i.e., as shown in the proof of \cref{lemma:additive-cost}, their degrees differ by at most $(1-4\beta)^2$ factor).
        Hence, an edge $z_e \in Z(v)$ appears $O(d(v))$ times across all $Z(\cdot)$, which concludes our analysis.
\end{proof}
Combining \cref{thm:approx-proof,lem:appro-singletons}, we derive our final approximation guarantee.
\begin{theorem} \label{thm:approx_main}
    Assume that $5 \beta + 2\lambda < 1/1.1$ and also assume that $\beta$ and $\lambda$ are positive constants. Then, with probability at least $1 - n^{-2}$
    \[
        \cost(\cref{alg:main}) \le O(\mathrm{OPT}) + O\rb{n \cdot \rb{\frac{\log{n}}{\eps} + \tfrac{\log^2 n \cdot \log(1/\delta)}{\min(1, \eps^2)}}}.
    \]
\end{theorem}

\subsection{Proof of \cref{corollary:same-cluster-different-clusters}}
\label{sec:corollary:same-cluster-different-clusters}
First, we prove the following claim.
\begin{lemma}\label{lemma:implications-betaL-betaU}
    Let $\myvec{\beta^L}, \myvec{\beta^U} \in \bbR_{\ge 0}^{V \times V}$ and $\myvec{\lambda^L}, \myvec{\lambda^U} \in \bbR_{\ge 0}^V$ such that $\myvec{\beta^U} \ge \myvec{\beta^L}$ and $\myvec{\lambda^U} \ge \myvec{\lambda^L}$. Let $\Erem$ be a subset of edges. Then, the following holds:
    \begin{enumerate}[(A)]
        \item\label{item:light} If $v$ is light in $\AlgCC(\myvec{\beta^U}, \myvec{\lambda^U}, \Erem)$, then $v$ is light in $\AlgCC(\myvec{\beta^L}, \myvec{\lambda^L}, \Erem)$.
        \item\label{item:heavy} If $v$ is heavy in $\AlgCC(\myvec{\beta^L}, \myvec{\lambda^L}, \Erem)$, then $v$ is heavy in $\AlgCC(\myvec{\beta^U}, \myvec{\lambda^U}, \Erem)$.
        \item\label{item:e-removed} If an edge $e$ is removed in $\AlgCC(\myvec{\beta^U}, \myvec{\lambda^U}, \Erem)$, then $e$ is removed in $\AlgCC(\myvec{\beta^L}, \myvec{\lambda^L}, \Erem)$ as well.
        \item\label{item:e-remains} If an edge $e$ remains in $\AlgCC(\myvec{\beta^L}, \myvec{\lambda^L}, \Erem)$, then $e$ remains in $\AlgCC(\myvec{\beta^U}, \myvec{\lambda^U}, \Erem)$ as well.
    \end{enumerate}
    
\end{lemma}
\begin{proof}
    Observe that $|N(u) \triangle N(v)| \le \myvec{\beta^L}_{u,v} \max\{d(u), d(v)\}$ implies $|N(u) \triangle N(v)| \le \myvec{\beta^U}_{u, v} \max\{d(u), d(v)\}$ as $\myvec{\beta^L}_{u,v} \le \myvec{\beta^U}_{u,v}$.
    Hence, if $u$ and $v$ are in agreement in $\AlgCC(\myvec{\beta^L}, \myvec{\lambda^L}, \Erem)$, then $u$ and $v$ are in agreement in $\AlgCC(\myvec{\beta^U}, \myvec{\lambda^U}, \Erem)$ as well. Similarly, if $u$ and $v$ are not in agreement in $\AlgCC(\myvec{\beta^U}, \myvec{\lambda^U}, \Erem)$, then $u$ and $v$ are not in agreement in $\AlgCC(\myvec{\beta^L}, \myvec{\lambda^L}, \Erem)$ as well.
    These observations immediately yield Properties~\ref{item:light} and \ref{item:heavy}. 
    
    To prove Properties~\ref{item:e-removed} and \ref{item:e-remains}, observe that an edge $e = \{u, v\}$ is removed from a graph if $u$ and $v$ are not in agreement, or if $u$ and $v$ are light, or if $e \in \Erem$. From our discussion above and from Property~\ref{item:light}, if $e$ is removed from $\AlgCC(\myvec{\beta^U}, \myvec{\lambda^U}, \Erem)$, then $e$ is removed from $\AlgCC(\myvec{\beta^L}, \myvec{\lambda^L}, \Erem)$ as well. On the other hand, $e \notin \Erem$ remains in $\AlgCC(\myvec{\beta^L}, \myvec{\lambda^L}, \Erem)$ if $u$ and $v$ are in agreement, and if $u$ or $v$ is heavy. Property~\ref{item:heavy} and our discussion about vertices in agreement imply Property~\ref{item:e-remains}.\footnote{Also, by contraposition, Property~\ref{item:e-remains} follows from Property~\ref{item:e-removed} and Property~\ref{item:heavy} follows from Property~\ref{item:light}.}
\end{proof}
As a corollary, we obtain the proof of \cref{corollary:same-cluster-different-clusters}.
\corollarysamecluster*
\begin{proof}
    \begin{enumerate}[(i)]
        \item\ Consider a path $P$ between $u$ and $v$ that makes them being in the same cluster/component in $\AlgCC(\myvec{\beta^L}, \myvec{\lambda^L}, \Erem)$. Then, by \cref{lemma:implications-betaL-betaU}~\ref{item:e-remains} $P$ remains in $\AlgCC(\myvec{\beta^U}, \myvec{\lambda^U}, \Erem)$ as well. Hence, $u$ and $v$ are in the same cluster of $\AlgCC(\myvec{\beta^U}, \myvec{\lambda^U}, \Erem)$.
    
    \item Follows from Property~\ref{item:u-and-v-same-cluster} by contraposition.
    \end{enumerate}
\end{proof}

\section{Lower bound}\label{sec:lower_bound}

In this section
we show that any private algorithm for correlation clustering must incur at least $\Omega(n)$ additive error in the approximation guarantee,
regardless of its multiplicative approximation ratio.
The following is a restatement of \cref{thm:lower_bound}.
\begin{theorem}\label{thm:lower_bound2}
    Let $\cA$ be an $(\epsilon,\delta)$-DP algorithm for correlation clustering on unweighted complete graphs,
    where $\epsilon \le 1$ and $\delta \le 0.1$.
    Then the expected cost of $\cA$ is at least $n/20$, even when restricted to instances whose optimal cost is $0$.
\end{theorem}
\begin{proof}
    \newcommand{\largervector}[3]{#3[#1 \leftarrow #2]}
    \newcommand{\marg}[2]{p^{(#1)}_{#2}}
    \newcommand{\marginal}[3]{\marg{#1}{\largervector{#1}{#2}{#3}}}
    Fix an even number $n = 2m$ of vertices and consider the fixed perfect matching $(1,2)$, $(3,4)$, $\ldots, (2m-1,2m)$.
    For every vector $\tau \in \{0,1\}^m$ we consider the instance $I_\tau$ obtained by having plus-edges $(2i-1,2i)$ for those $i=1,...,m$ where $\tau_i = 1$ (and minus-edges for $i$ with $\tau_i = 0$, as well as everywhere outside this perfect matching).
    Note that this instance is a complete unweighted graph and has optimal cost $0$.
    
    For $\tau \in \{0,1\}^m$ and $i \in \{1,...,m\}$ define $\marg{i}{\tau}$ to be the marginal probability that vertices $2i-1$ and $2i$ are in the same cluster when $\cA$ is run on the instance $I_\tau$.
    
    Finally,
    for $\sigma \in \{0,1\}^{m-1}$, $i \in \{1,...,m\}$ and $b \in \{0,1\}$ let $\largervector{i}{b}{\sigma}$ be the vector $\sigma$ with the bit $b$ inserted at the $i$-th position to obtain an $m$-dimensional vector (note that $\sigma$ is $(m-1)$-dimensional).
    Note that $I_{\largervector{i}{0}{\sigma}}$ and $I_{\largervector{i}{1}{\sigma}}$ are adjacent instances.
    Thus $(\epsilon,\delta)$-privacy gives
    \begin{equation}
        \label{eq:lb_privacy}
        \marginal{i}{1}{\sigma} \le e^\epsilon \cdot \marginal{i}{0}{\sigma} + \delta
    \end{equation}
    for all $i$ and $\sigma$.
    
    Towards a contradiction assume that $\cA$ achieves expected cost at most $0.05 n = 0.1 m$ on every instance $I_\tau$.
    In particular, the expected cost on the matching minus-edges is at most $0.1 m$, i.e.,
    \[ 0.1 m \ge \sum_{i : \tau_i = 0} \marg{i}{\tau} \,.  \]
    Summing this up over all vectors $\tau \in \{0,1\}^m$ we get
    \begin{equation}
        \label{eq:pi0sigma}
        2^m \cdot 0.1 m \ge \sum_{\tau \in \{0,1\}^m} \sum_{i : \tau_i = 0} \marg{i}{\tau} = \sum_i \sum_{\sigma \in \{0,1\}^{m-1}} \marginal{i}{0}{\sigma}
    \end{equation}
    and similarly since the expected cost on the matching plus-edges is at most $0.1 m$, we get
    \begin{align*}
    2^m \cdot 0.1 m
    &\ge \sum_{\tau \in \{0,1\}^m} \sum_{i : \tau_i = 1} (1 - \marg{i}{\tau}) \\
    &= \sum_i \sum_{\sigma \in \{0,1\}^{m-1}} (1 - \marginal{i}{1}{\sigma}) \\
    &\overset{\eqref{eq:lb_privacy}}{\ge} \sum_i \sum_{\sigma \in \{0,1\}^{m-1}} (1 - e^\epsilon \cdot \marginal{i}{0}{\sigma} - \delta) \\
    &= (1 - \delta) \cdot m \cdot 2^{m-1} - e^\epsilon \cdot \sum_i \sum_{\sigma \in \{0,1\}^{m-1}} \marginal{i}{0}{\sigma} \\
    &\overset{\eqref{eq:pi0sigma}}{\ge} (1 - \delta) \cdot m \cdot 2^{m-1} - e^\epsilon \cdot 2^m \cdot 0.1 m \\
    &\ge 0.45 \cdot m \cdot 2^m - 0.1 e \cdot 2^m \cdot m \,.
    \end{align*}
    Dividing by $2^m \cdot m$ gives $0.1 \ge 0.45 - 0.1e$,
    which is a contradiction.
\end{proof}

\bibliographystyle{alpha} 
\bibliography{references}

\newcommand{\etalchar}[1]{$^{#1}$}
\begin{thebibliography}{KCMNT08}

\bibitem[ACN08]{ailon2008aggregating}
Nir Ailon, Moses Charikar, and Alantha Newman.
\newblock Aggregating inconsistent information: ranking and clustering.
\newblock {\em Journal of the ACM (JACM)}, 55(5):1--27, 2008.

\bibitem[AHK{\etalchar{+}}09]{agrawal2009generating}
Rakesh Agrawal, Alan Halverson, Krishnaram Kenthapadi, Nina Mishra, and
  Panayiotis Tsaparas.
\newblock Generating labels from clicks.
\newblock In {\em Proceedings of the Second ACM International Conference on Web
  Search and Data Mining}, pages 172--181, 2009.

\bibitem[ARS09]{arasu2009large}
Arvind Arasu, Christopher R{\'e}, and Dan Suciu.
\newblock Large-scale deduplication with constraints using dedupalog.
\newblock In {\em 2009 IEEE 25th International Conference on Data Engineering},
  pages 952--963. IEEE, 2009.

\bibitem[AU19]{DBLP:conf/nips/AroraU19}
Raman Arora and Jalaj Upadhyay.
\newblock On differentially private graph sparsification and applications.
\newblock In Hanna~M. Wallach, Hugo Larochelle, Alina Beygelzimer, Florence
  d'Alch{\'{e}}{-}Buc, Emily~B. Fox, and Roman Garnett, editors, {\em Advances
  in Neural Information Processing Systems 32: Annual Conference on Neural
  Information Processing Systems 2019, NeurIPS 2019, December 8-14, 2019,
  Vancouver, BC, Canada}, pages 13378--13389, 2019.

\bibitem[AW22]{DBLP:conf/innovations/Assadi022}
Sepehr Assadi and Chen Wang.
\newblock Sublinear time and space algorithms for correlation clustering via
  sparse-dense decompositions.
\newblock In Mark Braverman, editor, {\em 13th Innovations in Theoretical
  Computer Science Conference, {ITCS} 2022, January 31 - February 3, 2022,
  Berkeley, CA, {USA}}, volume 215 of {\em LIPIcs}, pages 10:1--10:20. Schloss
  Dagstuhl - Leibniz-Zentrum f{\"{u}}r Informatik, 2022.

\bibitem[BBC04]{bansal2004correlation}
Nikhil Bansal, Avrim Blum, and Shuchi Chawla.
\newblock Correlation clustering.
\newblock {\em Machine learning}, 56(1):89--113, 2004.

\bibitem[BBDS12]{DBLP:conf/focs/BlockiBDS12}
Jeremiah Blocki, Avrim Blum, Anupam Datta, and Or~Sheffet.
\newblock The johnson-lindenstrauss transform itself preserves differential
  privacy.
\newblock In {\em 53rd Annual {IEEE} Symposium on Foundations of Computer
  Science, {FOCS} 2012, New Brunswick, NJ, USA, October 20-23, 2012}, pages
  410--419. {IEEE} Computer Society, 2012.

\bibitem[BBDS13]{DBLP:conf/innovations/BlockiBDS13}
Jeremiah Blocki, Avrim Blum, Anupam Datta, and Or~Sheffet.
\newblock Differentially private data analysis of social networks via
  restricted sensitivity.
\newblock In Robert~D. Kleinberg, editor, {\em Innovations in Theoretical
  Computer Science, {ITCS} '13, Berkeley, CA, USA, January 9-12, 2013}, pages
  87--96. {ACM}, 2013.

\bibitem[BCSZ18]{DBLP:conf/focs/BorgsCSZ18}
Christian Borgs, Jennifer~T. Chayes, Adam~D. Smith, and Ilias Zadik.
\newblock Revealing network structure, confidentially: Improved rates for
  node-private graphon estimation.
\newblock In Mikkel Thorup, editor, {\em 59th {IEEE} Annual Symposium on
  Foundations of Computer Science, {FOCS} 2018, Paris, France, October 7-9,
  2018}, pages 533--543. {IEEE} Computer Society, 2018.

\bibitem[BDL{\etalchar{+}}17]{balcan}
Maria{-}Florina Balcan, Travis Dick, Yingyu Liang, Wenlong Mou, and Hongyang
  Zhang.
\newblock Differentially private clustering in high-dimensional euclidean
  spaces.
\newblock In Doina Precup and Yee~Whye Teh, editors, {\em Proceedings of the
  34th International Conference on Machine Learning, {ICML}}, volume~70 of {\em
  Proceedings of Machine Learning Research}, pages 322--331. {PMLR}, 2017.

\bibitem[BEK21]{bun2021differentially}
Mark Bun, Marek Elias, and Janardhan Kulkarni.
\newblock Differentially private correlation clustering.
\newblock In {\em International Conference on Machine Learning}, pages
  1136--1146. PMLR, 2021.

\bibitem[BGU13]{bonchi2013overlapping}
Francesco Bonchi, Aristides Gionis, and Antti Ukkonen.
\newblock Overlapping correlation clustering.
\newblock {\em Knowledge and information systems}, 35(1):1--32, 2013.

\bibitem[BNSV15]{DBLP:conf/focs/BunNSV15}
Mark Bun, Kobbi Nissim, Uri Stemmer, and Salil~P. Vadhan.
\newblock Differentially private release and learning of threshold functions.
\newblock In Venkatesan Guruswami, editor, {\em {IEEE} 56th Annual Symposium on
  Foundations of Computer Science, {FOCS} 2015, Berkeley, CA, USA, 17-20
  October, 2015}, pages 634--649. {IEEE} Computer Society, 2015.

\bibitem[CALM{\etalchar{+}}21]{cohen2021correlation}
Vincent Cohen-Addad, Silvio Lattanzi, Slobodan Mitrovi{\'c}, Ashkan
  Norouzi-Fard, Nikos Parotsidis, and Jakub Tarnawski.
\newblock Correlation clustering in constant many parallel rounds.
\newblock {\em arXiv preprint arXiv:2106.08448}, 2021.

\bibitem[CGKM21]{badih_local}
Alisa Chang, Badih Ghazi, Ravi Kumar, and Pasin Manurangsi.
\newblock Locally private k-means in one round.
\newblock {\em CoRR}, abs/2104.09734, 2021.

\bibitem[CGW05]{charikar2005clustering}
Moses Charikar, Venkatesan Guruswami, and Anthony Wirth.
\newblock Clustering with qualitative information.
\newblock {\em Journal of Computer and System Sciences}, 71(3):360--383, 2005.

\bibitem[CKP08]{chakrabarti2008graph}
Deepayan Chakrabarti, Ravi Kumar, and Kunal Punera.
\newblock A graph-theoretic approach to webpage segmentation.
\newblock In {\em Proceedings of the 17th international conference on World
  Wide Web}, pages 377--386, 2008.

\bibitem[CMSY15]{chawla2015near}
Shuchi Chawla, Konstantin Makarychev, Tselil Schramm, and Grigory Yaroslavtsev.
\newblock Near optimal lp rounding algorithm for correlationclustering on
  complete and complete k-partite graphs.
\newblock In {\em Proceedings of the forty-seventh annual ACM symposium on
  Theory of computing}, pages 219--228, 2015.

\bibitem[CNX20]{anamayclustering}
Anamay Chaturvedi, Huy~L. Nguyen, and Eric Xu.
\newblock Differentially private k-means clustering via exponential mechanism
  and max cover.
\newblock {\em CoRR}, abs/2009.01220, 2020.

\bibitem[CSX12]{chen2012clustering}
Yudong Chen, Sujay Sanghavi, and Huan Xu.
\newblock Clustering sparse graphs.
\newblock In {\em Proceedings of the 25th International Conference on Neural
  Information Processing Systems-Volume 2}, pages 2204--2212, 2012.

\bibitem[DEFI06]{demaine2006correlation}
Erik~D Demaine, Dotan Emanuel, Amos Fiat, and Nicole Immorlica.
\newblock Correlation clustering in general weighted graphs.
\newblock {\em Theoretical Computer Science}, 361(2-3):172--187, 2006.

\bibitem[DMNS06]{DBLP:conf/tcc/DworkMNS06}
Cynthia Dwork, Frank McSherry, Kobbi Nissim, and Adam~D. Smith.
\newblock Calibrating noise to sensitivity in private data analysis.
\newblock In Shai Halevi and Tal Rabin, editors, {\em Theory of Cryptography,
  Third Theory of Cryptography Conference, {TCC} 2006, New York, NY, USA, March
  4-7, 2006, Proceedings}, volume 3876 of {\em Lecture Notes in Computer
  Science}, pages 265--284. Springer, 2006.

\bibitem[DR{\etalchar{+}}14]{dwork2014algorithmic}
Cynthia Dwork, Aaron Roth, et~al.
\newblock The algorithmic foundations of differential privacy.
\newblock {\em Found. Trends Theor. Comput. Sci.}, 9(3-4):211--407, 2014.

\bibitem[DRV10]{DBLP:conf/focs/DworkRV10}
Cynthia Dwork, Guy~N. Rothblum, and Salil~P. Vadhan.
\newblock Boosting and differential privacy.
\newblock In {\em 51th Annual {IEEE} Symposium on Foundations of Computer
  Science, {FOCS} 2010, October 23-26, 2010, Las Vegas, Nevada, {USA}}, pages
  51--60, 2010.

\bibitem[Dwo06]{DBLP:conf/icalp/Dwork06}
Cynthia Dwork.
\newblock Differential privacy.
\newblock In {\em Automata, Languages and Programming, 33rd International
  Colloquium, {ICALP} 2006, Venice, Italy, July 10-14, 2006, Proceedings, Part
  {II}}, pages 1--12, 2006.

\bibitem[EKKL20]{DBLP:conf/soda/EliasKKL20}
Marek Eli{\'{a}}s, Michael Kapralov, Janardhan Kulkarni, and Yin~Tat Lee.
\newblock Differentially private release of synthetic graphs.
\newblock In {\em Proceedings of the 2020 {ACM-SIAM} Symposium on Discrete
  Algorithms, {SODA} 2020, Salt Lake City, UT, USA, January 5-8, 2020}, pages
  560--578, 2020.

\bibitem[FHS21]{DBLP:journals/corr/abs-2106-00508}
Alireza Farhadi, MohammadTaghi Hajiaghayi, and Elaine Shi.
\newblock Differentially private densest subgraph.
\newblock {\em CoRR}, abs/2106.00508, 2021.

\bibitem[GG05]{giotis2005correlation}
Ioannis Giotis and Venkatesan Guruswami.
\newblock Correlation clustering with a fixed number of clusters.
\newblock {\em arXiv preprint cs/0504023}, 2005.

\bibitem[GKM20]{badih_approximation}
Badih Ghazi, Ravi Kumar, and Pasin Manurangsi.
\newblock Differentially private clustering: Tight approximation ratios.
\newblock In Hugo Larochelle, Marc'Aurelio Ranzato, Raia Hadsell,
  Maria{-}Florina Balcan, and Hsuan{-}Tien Lin, editors, {\em Advances in
  Neural Information Processing Systems}, 2020.

\bibitem[GLM{\etalchar{+}}10]{DBLP:conf/soda/GuptaLMRT10}
Anupam Gupta, Katrina Ligett, Frank McSherry, Aaron Roth, and Kunal Talwar.
\newblock Differentially private combinatorial optimization.
\newblock In {\em Proceedings of the Twenty-First Annual {ACM-SIAM} Symposium
  on Discrete Algorithms, {SODA} 2010, Austin, Texas, USA, January 17-19,
  2010}, pages 1106--1125, 2010.

\bibitem[GRU12]{DBLP:conf/tcc/GuptaRU12}
Anupam Gupta, Aaron Roth, and Jonathan~R. Ullman.
\newblock Iterative constructions and private data release.
\newblock In Ronald Cramer, editor, {\em Theory of Cryptography - 9th Theory of
  Cryptography Conference, {TCC} 2012, Taormina, Sicily, Italy, March 19-21,
  2012. Proceedings}, volume 7194 of {\em Lecture Notes in Computer Science},
  pages 339--356. Springer, 2012.

\bibitem[HLMJ09]{DBLP:conf/icdm/HayLMJ09}
Michael Hay, Chao Li, Gerome Miklau, and David~D. Jensen.
\newblock Accurate estimation of the degree distribution of private networks.
\newblock In Wei Wang, Hillol Kargupta, Sanjay Ranka, Philip~S. Yu, and Xindong
  Wu, editors, {\em {ICDM} 2009, The Ninth {IEEE} International Conference on
  Data Mining, Miami, Florida, USA, 6-9 December 2009}, pages 169--178. {IEEE}
  Computer Society, 2009.

\bibitem[KCMNT08]{kalashnikov2008web}
Dmitri~V Kalashnikov, Zhaoqi Chen, Sharad Mehrotra, and Rabia Nuray-Turan.
\newblock Web people search via connection analysis.
\newblock {\em IEEE Transactions on Knowledge and Data Engineering},
  20(11):1550--1565, 2008.

\bibitem[KNRS13]{DBLP:conf/tcc/KasiviswanathanNRS13}
Shiva~Prasad Kasiviswanathan, Kobbi Nissim, Sofya Raskhodnikova, and Adam~D.
  Smith.
\newblock Analyzing graphs with node differential privacy.
\newblock In Amit Sahai, editor, {\em Theory of Cryptography - 10th Theory of
  Cryptography Conference, {TCC} 2013, Tokyo, Japan, March 3-6, 2013.
  Proceedings}, volume 7785 of {\em Lecture Notes in Computer Science}, pages
  457--476. Springer, 2013.

\bibitem[KRSY11]{DBLP:journals/pvldb/KarwaRSY11}
Vishesh Karwa, Sofya Raskhodnikova, Adam~D. Smith, and Grigory Yaroslavtsev.
\newblock Private analysis of graph structure.
\newblock {\em Proc. {VLDB} Endow.}, 4(11):1146--1157, 2011.

\bibitem[Liu22]{Daogao2022}
Daogao Liu.
\newblock Better private algorithms for correlation clustering.
\newblock {\em CoRR}, arXiv:2202.10747, 2022.

\bibitem[LS20]{clustering_with_convergence}
Zhigang Lu and Hong Shen.
\newblock Differentially private k-means clustering with guaranteed
  convergence.
\newblock {\em CoRR}, abs/2002.01043, 2020.

\bibitem[MT07]{conf/focs/McSherryT07}
Frank McSherry and Kunal Talwar.
\newblock Mechanism design via differential privacy.
\newblock In {\em 48th Annual {IEEE} Symposium on Foundations of Computer
  Science {(FOCS} 2007), October 20-23, 2007, Providence, RI, USA,
  Proceedings}, pages 94--103, 2007.

\bibitem[NV21]{DBLP:conf/icml/NguyenV21}
Dung Nguyen and Anil Vullikanti.
\newblock Differentially private densest subgraph detection.
\newblock In Marina Meila and Tong Zhang, editors, {\em Proceedings of the 38th
  International Conference on Machine Learning, {ICML} 2021, 18-24 July 2021,
  Virtual Event}, volume 139 of {\em Proceedings of Machine Learning Research},
  pages 8140--8151. {PMLR}, 2021.

\bibitem[RHMS09]{DBLP:conf/pods/RastogiHMS09}
Vibhor Rastogi, Michael Hay, Gerome Miklau, and Dan Suciu.
\newblock Relationship privacy: output perturbation for queries with joins.
\newblock In Jan Paredaens and Jianwen Su, editors, {\em Proceedings of the
  Twenty-Eigth {ACM} {SIGMOD-SIGACT-SIGART} Symposium on Principles of Database
  Systems, {PODS} 2009, June 19 - July 1, 2009, Providence, Rhode Island,
  {USA}}, pages 107--116. {ACM}, 2009.

\bibitem[Swa04]{swamy2004correlation}
Chaitanya Swamy.
\newblock Correlation clustering: maximizing agreements via semidefinite
  programming.
\newblock In {\em SODA}, volume~4, pages 526--527. Citeseer, 2004.

\bibitem[US19]{DBLP:conf/nips/UllmanS19}
Jonathan~R. Ullman and Adam Sealfon.
\newblock Efficiently estimating erdos-renyi graphs with node differential
  privacy.
\newblock In Hanna~M. Wallach, Hugo Larochelle, Alina Beygelzimer, Florence
  d'Alch{\'{e}}{-}Buc, Emily~B. Fox, and Roman Garnett, editors, {\em Advances
  in Neural Information Processing Systems 32: Annual Conference on Neural
  Information Processing Systems 2019, NeurIPS 2019, December 8-14, 2019,
  Vancouver, BC, Canada}, pages 3765--3775, 2019.

\end{thebibliography}

\end{document}